\documentclass{article}

\usepackage[utf8]{inputenc} 
\usepackage[T1]{fontenc}
\usepackage{lmodern}
\usepackage{hyperref}       
\usepackage{url}            
\usepackage{booktabs}       
\usepackage{amsfonts}       
\usepackage[numbers]{natbib}
\usepackage{nicefrac}       
\usepackage{microtype}      
\usepackage{algorithm,algpseudocode}
\usepackage{amsfonts}
\usepackage{amsmath}
\usepackage{amsthm}
\usepackage{amssymb}
\usepackage[margin=1.1in]{geometry}
\usepackage{cases}
\usepackage[title]{appendix}
\usepackage{bm}
\usepackage{courier}
\usepackage[usenames,dvipsnames]{color}
\usepackage{enumitem}
\usepackage{graphicx}
\usepackage{url}
\usepackage{caption}
\usepackage{subcaption}
\usepackage{wrapfig}
\usepackage{multirow}
\usepackage{multicol}
\usepackage[dvipsnames]{xcolor}
\hypersetup{
	colorlinks,
	linkcolor={red!80!black},
	citecolor={blue!50!black},
	urlcolor={blue!80!black}
}

\input{./Definitions}

\title{Feature space sketching for logistic regression}

\usepackage{parskip}

\hypersetup{
	colorlinks,
	linkcolor={red!80!black},
	citecolor={blue!50!black},
	urlcolor={blue!80!black}
}

\author{
Gregory Dexter%
\footnote{Department of Computer Science, Purdue University,
West Lafayette, IN, USA, \texttt{\{gdexter, rajivak, jraheel, pdrineas\}@purdue.edu}}
\and
Rajiv Khanna%
\footnotemark[1]
\and
Jawad Raheel%
\footnotemark[1]
\and
Petros Drineas%
\footnotemark[1]
}

\begin{document}
\date{}
\maketitle

\begin{abstract}
We present novel bounds for coreset construction, feature selection, and dimensionality reduction for logistic regression. All three approaches can be thought of as \textit{sketching} the logistic regression inputs. On the coreset construction front, we resolve open problems from prior work and present novel bounds for the complexity of coreset construction methods. On the feature selection and dimensionality reduction front, we initiate the study of forward error bounds for logistic regression. Our bounds are tight up to constant factors and our forward error bounds can be extended to Generalized Linear Models. 
\end{abstract}

\section{Introduction}

Logistic regression is an indispensable tool in data analysis, dating back to at least the early 19th century. It was initially used to make predictions in social science and chemistry applications~\cite{PFV1838,PFV1845,Cramer2003}. Over the past 200 years it has been applied to all data-driven scientific domains, from economics and the social sciences to physics, medicine, and biology. At a high level, the (binary) logistic regression model is a statistical abstraction that models the probability of one of two alternatives or classes by expressing the log-odds (the logarithm of the odds) for the class as a linear combination of one or more predictor variables. 

Formally, consider the following regularized regression problem:
\begin{align}
    \beta_d &= \argmin_{\beta \in \R^d}\Lcal(\beta)  
    = \argmin_{\beta \in \R^d} \sum_{i=1}^n \log(1 + e^{-\yb_i\xb_i^T\beta}) + \frac{\lambda}{2}\|\beta\|_2^2,\label{eqn:blrr}
\end{align}
where $\Xb \in \R^{n \times d}$ is the data matrix ($n$ points in $\R^d$, with $\xb_i^T$ being the rows of $\Xb$); $\yb \in \{-1,1\}^n$ is the vector of labels whose entries are the $\yb_i$; and $\lambda > 0$ is the regularization parameter.

Over the past 25 years, matrix sketching and Randomized Linear Algebra~\cite{Drineas2016,Drineas2017,Halko2011,KannanV17,Martinsson2020} have advocated the use of sketching and sampling techniques to \textit{reduce} the dimensionality of the input matrix $\Xb$ (and the response vector $\yb$) in a variety of optimization problems in order to reduce running times and improve interpretability with a bounded accuracy loss. For example, if one considers the much more common setting of \textit{linear} regression (least-squares or $\ell_2$ regression and, more generally, $\ell_p$ regression), there has been an extensive body of work studying the effect of matrix sketching and sampling for \textit{preconditioning, coreset construction, feature selection, dimensionality reduction, etc.} for (regularized or not) under- and over-constrained regression problems~\cite{avron2010blendenpik,chowdhury2018iterative,clarkson2009numerical,DMM06,Drineas2011,Sarlos2006,Clarkson2013a,Clarkson2013b,Clarkson2016,Avron2016,Pilanci2014,Pilanci2016,Derezinski2017,Derezinski2021}. At a high level, matrix sketching and sampling can be used to speed up linear regression in theory and in practice, either by preconditioning the input and then applying standard solvers or by constructing and solving a ``smaller'' regression problem, with little loss in accuracy. Additionally, from an interpretability perspective, we know how to construct coresets or select/extract features, thus reducing the dimensionality of the constraint space or the feature space, so that solving the ``induced'' problem results in small accuracy loss. In many cases, tight upper and lower bounds are known for the aforementioned operations and, indeed, solving linear regression problems has been a massive success story for matrix sketching and Randomized Linear Algebra, both in theory and in practice.

Despite the central importance of logistic regression in statistics and machine learning, and unlike linear regression, relatively little is known about how the method behaves when matrix sketching and sampling are applied to its input.~\citet{munteanu2018coresets,pmlr-v139-munteanu21a} initiated the study of \textit{coresets} for logistic regression, and~\citet{mai2021coresets} provided the current state-of-the-art bounds for this problem. In these works, coreset constructions reduce the number of training points in the logistic regression problem, while guaranteeing that the reduced problem satisfies provable bounds on the \textit{in-sample} logistic loss. The reduction in the number of points is typically performed via sampling algorithms, where a subset of the original points is selected with respect to carefully constructed probability distributions, such as the Lewis scores. Coreset construction methods can be thought of as \textit{sketching} the input matrix $\Xb$ from the left, using a sketching matrix of canonical vectors, where each vector indicates an input point (row of $\Xb$) to be included in the coreset. It is worth noting that~\citet{pmlr-v139-munteanu21a} even considered oblivious sketching methods to construct coresets that are not a subset of the original set of points, but consist of linear combinations of the original points, thus moving beyond the canonical vector paradigm for the sketching matrix.

Even less is known for reducing the dimensionality of the $d$-dimensional feature space of the logistic regression problem of eqn.~(\ref{eqn:blrr}). To the best of our knowledge, only feature selection methods have been considered in prior work~\cite{Lozano2011,elenberg2018restricted} (see Section~\ref{sxn:related} for details). Again, feature selection methods can be thought of as sketching the input matrix $\Xb$ from the right, using a sketching matrix of canonical vectors, where each vector indicates a feature (column of $\Xb$) to be selected. Other dimensionality reduction and feature extraction methods, which use a smaller subset of \textit{artificial} instead of \textit{actual} features to operate on, have not been studied at all in prior work.

Our work makes progress in both fronts, resolving open problems from prior work and presenting novel bounds for the complexity of coreset construction methods. Importantly, we initiate the study of \textit{forward error bounds} for dimensionality reduction methods (including both feature selection and feature extraction) for the feature space of logistic regression. Our work presents tight bounds for \textit{linear} dimensionality reduction methods for logistic regression. Interestingly, our forward error bounds are also applicable to the much more general setting of Generalized Linear Models (GLMs). 

\subsection{Our contributions}\label{section:contributions}

\vspace{0.04in}\noindent\textbf{Coreset construction (row sketching).} In Section \ref{sxn:logistic_loss}, we consider the problem of approximating the logistic loss function $\Lcal(\beta)$ by a different function $\Lcaltil(\beta)$. This is the setting of prior work on coresets for logistic regression~\cite{munteanu2018coresets, mai2021coresets}, where the function $\Lcaltil(\beta)$ is evaluated, for example, only on a subset of the input points $\xb_i$ (the coreset). Just like prior work, we are interested in bounding the approximation error 
$|\Lcal(\beta) - \Lcaltil(\beta)|,$
as a function of $\beta \in \R^d$ and we present a number of novel results in this context.

First, we lower bound the worst-case space complexity of \emph{any} data structure that approximates logistic loss to $\epsilon$-relative error. Coresets for logistic regression introduced in prior work are special cases of a low-space data structure giving an $\epsilon$-relative error approximation to logistic loss. Our bound implies that the construction given by \citet{mai2021coresets} has optimal size up to a $\Ocaltil(d\cdot\mu_\yb(\Xb)^2)$ factor, where $\mu_\yb(\Xb)$ is a data-dependent complexity measure (Definition~\ref{def:complexity_measure}). Bridging this gap would imply that coreset constructions are actually optimal among all possible data structures that attempt to reduce the number of points (or even the number of bits used to sketch the input) that logistic regression operates on. Our bound is a first step towards this objective.

Our second contribution is an efficient linear programming formulation for computing the data complexity measure $\mu_\yb(\Xb)$ of Definition~\ref{def:complexity_measure}, which was previously conjectured to be hard to compute by~\citet{munteanu2018coresets}. This quantity is a critical component of existing theoretical results bounding the necessary size of a coreset for logistic regression. Therefore, an efficient and accurate method to evaluate this complexity measure was a critical missing piece and paves the road towards an improved empirical validation of coreset constructions for logistic regression.

Finally, we show that replacing the input matrix $\Xb$ by any other matrix $\tilde{\Xb}$ such that the two-norm\footnote{Recall that the two-norm, also called the spectral norm, of a matrix $\Xb \in \mathbb{R}^{n \times d}$ is $\|\Xb\|_2 = \max_{\xb \in \mathbb{R}^d,\ \|\xb\|_2=1} \|\Xb\xb\|_2 = \max_{\yb \in \mathbb{R}^n,\ \|\yb\|_2=1} \|\yb^T\Xb\|_2$.} error
$\|\Xb - \tilde{\Xb}\|_2$
is bounded can be used to obtain an additive error approximation to the logistic loss. We also prove that this additive error guarantee is tight in the worst case. Low-rank approximations to $\Xb$ are often used instead of $\Xb$ in logistic regression, either to reduce the intrinsic dimensionality\footnote{Reducing the intrinsic dimensionality (rank) of the input results in storage savings, since one can store only a few singular values and singular vectors instead of the whole matrix. Also, it could result in improved running times, since many numerical algorithms take advantage of low-rank matrices.} of the input or to denoise it. Such low-rank approximations do result in a bounded two-norm approximation: It is well-known that setting $\tilde{\Xb} = \Xb_k$, where $\Xb_k$ is the best rank-$k$ approximation to $\Xb$ computed, for example, via the Singular Value Decomposition (SVD)\footnote{Randomized Linear Algebra algorithms can be used to approximate the best low-rank approximation to $\Xb$ with bounded accuracy loss~\cite{Halko2011,Drineas2017,Drineas2018,kishore2017literature}.}, satisfies $\|\Xb - \Xb_k\|_2 =\sigma_{k+1}(\Xb)$\footnote{$\sigma_{k+1}(\Xb)$ is the $(k+1)$st singular value of $\Xb$.}. Somewhat surprisingly, such results were not known for logistic regression prior to our work.

\vspace{0.04in}\noindent\textbf{Reducing the feature-space dimensionality (column sketching).} 
Our most fundamental contribution is the development of novel, tight bounds for sketching the feature space (columns) of regularized logistic regression problems. More specifically, given the data matrix $\Xb \in \R^{n \times d}$ of eqn.~(\ref{eqn:blrr}), we study the error induced by projecting the columns (features) of $\Xb$ to an arbitrary $k$-dimensional subspace (with $k \ll d$) spanned by the columns of a $k \times d$ matrix, $\Pb_k$. Formally, consider the following, \textit{sketched} regression problem:
\begin{align}
    \beta_k &= \argmin_{\beta \in \R^k} \Lcal_k(\beta) 
    = \argmin_{\beta \in \R^k} \sum_{i=1}^n \log(1 + e^{-\yb_i\xb_i^T\Pb_k\beta}) + \frac{\lambda}{2} \|\beta\|_2^2,\label{eqn:sketchedregression1}
\end{align}
where $\Pb_k \in \R^{d \times k}$, $k \ll d$, has orthonormal columns; all other quantities are as in eqn.~(\ref{eqn:blrr}). This family of dimensionality reducing transformations is relevant to understanding several important topics for sketched logistic regression:
\begin{itemize}
    \item Any feature selection method for logistic regression problems can be modelled as a matrix $\Pb_k$ whose columns are standard basis vectors. This includes popular methods that interpret the output of logistic regression by keeping only the $k$ largest magnitude coefficients of the learned parameter vector $\beta$~\cite{Wu2009,Held2016,Prive2019} and attempt to interpret the associated features.
    \item Dimensionality reduction methods for logistic regression that project the feature matrix on its top $k$ right singular vectors, thus using those singular vectors as eigenfeatures. Empirical evidence suggests that reducing the features of a data set to its top principal components can greatly reduce the computational and memory burden of solving large-scale logistic regression problems with limited reduction in accuracy. This can be modeled in our framework as the case where $\Pb_k$ are the top $k$ right singular vectors of $\Xb$. 
\end{itemize}
As we shall see in Section~\ref{sxn:logistic_loss} and Theorem~\ref{thm:logistic_loss_lower}, there are barriers to getting relative error guarantees on the logistic loss when sketching the high-dimensional feature space of a data set, thus effectively reducing the rank of the input matrix. Motivated by such lower bounds, we consider the \emph{forward error} when sketching the feature space, namely the two-norm difference between the optimal parameters of the original and sketched problems: 
\begin{equation}\label{eqn:pd111}
\|\Pb_k\beta_k - \beta_d\|_2^2.
\end{equation}
We prove the following complementary upper and lower bounds on the forward error of eqn.~(\ref{eqn:pd111}). The bounds hold for \emph{any} valid tuple $(\Xb, \yb, \Pb_k, \lambda)$:
\begin{align*}
     \frac{4}{ \|\Xb\|_2^2 + \lambda} \cdot \Phi(\Xb,\yb,\Pb_k)
     \leq \|\Pb_k\beta_k - \beta_d\|_2^2
     \leq\frac{2}{\lambda}\Phi(\Xb,\yb,\Pb_k).
\end{align*}
In the above, $\Phi(\Xb,\yb,\Pb_k)$ is the following function\footnote{We omit $\beta_k$ and $\beta_d$ from the definition of $\Phi(\cdot)$ since they can be computed given the other inputs.}:
$$\Phi(\Xb,\yb,\Pb_k) = \beta_d^T(\Ib - \Pb_k\Pb_k^T) \Xb^T\Db_\yb \wb(\Pb_k\beta_k),$$
where $\Db_\yb \in \R^{n \times n}$ denotes the diagonal matrix whose entries are $\yb_i$; $\wb(\beta) \in \R^n$ has entries $\wb(\beta)_i = \sigma(-\yb_i \xb_i^T\beta)$, for all $i=1\ldots n$; and $\sigma(u)$ denotes the logistic function $\sigma(u) = \nicefrac{1}{1 + e^{-u}}$.  
Note that our bounds are \textit{tight up to a constant factor} when $\|\Xb\|_2^2 \leq C\lambda$ for some constant $C > 0$.

To better understand the above bound, one could interpret $\wb(\Pb_k\beta_k)$ as a vector describing the distribution of the $i$-th label under the model given by $\Pb_k\beta_k$. The norm can be bounded naively by $\sqrt{n}$. Additional structural assumptions could restrict the norm further. Under generative model assumptions, one could bound the term $\sigma(\xb_i^T \Pb_k \Pb_k^T \beta_d)$ to probabilistically bound the entries of $\wb(\Pb_k\beta_k)$, and hence its norm, more sharply.

It is worth noting that forward error bounds are a natural accuracy metric in numerical analysis and theoretical computer science. For example, for linear regression, forward error bounds in the context of sensitivity analysis, to bound the accuracy of iterative approximation algorithms, or to analyze randomized linear regression algorithms, have been a topic of intense research for many years (see the works of~\citet{Higham2002,Woodruff2014,Drineas2017} and references therein). In the context of logistic regression, forward error bounds are optimization-based guarantees that are considered to be a precursor for generalization errors which typically require additional statistical assumptions ~\cite{loh2015regularized,negahban2012unified,elenberg2018restricted}. We do emphasize that related prior work also focused on error metrics such as in-sample logistic loss~\cite{munteanu2018coresets,mai2021coresets}, that have no known connection to the generalization error for logistic regression.

\vspace{0.04in}\noindent\textbf{Extensions.} We generalize our results in three different ways.  In Appendix~\ref{section:generative_model}, we show that generative assumptions on the data set $(\Xb, \yb)$ from prior work are enough to prove our forward error bounds \emph{without} any explicit regularization (i.e. with $\lambda=0$). In Appendix~ \ref{section:glm}, we show that our upper and lower forward error bounds generalize to a wide class of generalized linear models (GLMs) including linear and multinomial regression.

Finally, in Appendix~\ref{section:information_theoretic_bound}, we show that our upper bound can be naturally interpreted from an information theoretic perspective under the typical modeling assumptions of logistic regression. More specifically, given the matrix of $n$ points $\Xb$ and the logistic regression parameter $\beta$, let $\ybtil \in\{-1, 1\}^n$ be a random binary vector where the distribution of the $i$-th coordinate is the distribution given by the linear-log odds model parameterized by $\beta$. If we additionally have a true label vector $\yb\in\{-1,1\}^n$, then we can define a new random vector $\ybhat \in \{-1,1\}^n$ such that $\ybhat = \yb \cdot \ybtil$, where the multiplication is entry-wise.  Intuitively, $\PP(\ybhat_i = 1)$ is the probability that a label drawn from the linear log-odds model agrees with the true label, and this measures how well the model fits the true labels. We prove that if $\Dcal_p$ is the distribution of $\ybhat$ induced by $\beta = (\Ib - \Pb_k\Pb_k^T)\beta_d$ and $\Dcal_q$ is the distribution of $\ybhat$ induced by $\beta = \Pb_k\beta_k$, then 
$$\|\Pb_k\beta_k - \beta_d\|_2^2 \leq \frac{2}{\lambda} H(\Dcal_p, \Dcal_q),$$ 
where $H$ denotes the cross-entropy between the two distributions.

\subsection{Related Work}\label{sxn:related}

A line of research that motivated our work focused on the development of coresets for logistic regression. This was initiated by~\citet{munteanu2018coresets}, and the current state-of-the-art bounds for coreset construction for logistic regression appeared in~\cite{mai2021coresets}.  We already discussed the contributions of these two papers earlier in the introduction. 
Another line of research that is relevant to our work is that of feature selection for logistic regression. In this setup, the goal is to select $k$ features out of $d$, which can be construed as restricting the sketching matrix $\mathbf{P}_k$ to specifically be a selection matrix. Towards this goal,~\citet{elenberg2018restricted} provide approximation and recovery guarantees for performance of greedy selection of features vis-{\`a}-vis the best possible \emph{sparse} solution for \emph{general} functions. On the other hand, our results quantify recovery-like guarantees for the parameter $\beta$ for \emph{any} given $\mathbf{P}_k$ and \emph{any} possible $\beta_d$ but specifically for logistic regression. Since we generalize in different ways, our bounds are not directly comparable with theirs. Even if we specialize their results for logistic regression and choose our $\Pb_k$ to select the features selected by the greedy algorithm for the special case when $\beta_d$ is inherently $k$-sparse, our Theorem~\ref{thm:forward_error_upper_bound} still is incomparable to~\citet{elenberg2018restricted}[Theorem 6] because the two theorems upper bound slightly different quantities.~\citet{Lozano2011} also provide guarantees for logistic regression which are similar to those of~\citet{elenberg2018restricted} but under group-sparsity constraints, and are incomparable to our results.

\section{Approximating Logistic Loss}\label{sxn:logistic_loss}

In this section, we consider approximations to the \emph{logistic loss} function. This type of error guarantees were previously provided by a body of work on coresets for logistic regression; see the works of~\citet{munteanu2018coresets, mai2021coresets} for state-of-the-art results and bounds and references therein for motivation and prior work on this topic.

We begin by proving a general lower bound on the space complexity of \emph{any} data structure which approximates logistic loss to $\epsilon$-relative error for every parameter vector $\beta \in \R^d$. Our lower bound shows that the coreset construction of \citet{mai2021coresets} has optimal space complexity up to an $\Ocaltil(d\cdot \mu_\yb(\Xb)^2)$ factor, where $\mu_\yb(\Xb)^2$ is a data complexity measure introduced by~\citet{munteanu2018coresets} (see also Definition~\ref{def:complexity_measure} below).  Furthermore, we provide an efficient linear programming formulation to compute this data complexity measure on real data sets, \textit{refuting an earlier conjecture that the factor was hard to compute}~\cite{munteanu2018coresets}.

Next, we show that any low-rank approximation $\Xbbar$ of the data matrix $\Xb$ can be used to approximate the logistic loss function $\Lcal(\beta)$ up to a $\sqrt{n}\|\Xb - \Xbbar\|_2\|\beta\|_2$ additive error. The factor $\|\Xb - \Xbbar\|_2\|\beta\|_2$ is the spectral norm (or two-norm) error of the low-rank approximation and we also prove that this bound is tight in the worst case. Low rank approximations are commonly used to reduce the time and space complexity of numerical algorithms, especially in settings where the data matrix $\Xb$ is numerically low-rank or has a decaying spectrum of singular values.

In this section, we assume without loss of generality that $\yb_i = -1$ for all $i=1\ldots n$. Recall that any logistic regression problem can be transformed to this standard form by multiplying both $\Xb$ and $\yb$ by $-\Db_\yb$ (see Section~\ref{section:contributions} for notation), in which case the logistic loss of the original problem is equal to that of the transformed problem for all $\beta\in \R^d$.  Furthermore, we assume in this section that the logistic loss has no regularization, i.e. $\lambda = 0$. This is because increasing the regularization parameter will only increase the quality of the approximation to the logistic loss in terms of relative error, as the regularization penalty can be computed exactly. 

\subsection{General space complexity lower bound}

We lower bound the space complexity of any data structure that approximates logistic loss to $\epsilon$-relative error. More specifically, assume we are given a logistic regression problem specified by $\Xb$ and $\yb$ with bounded bit complexity. We consider the problem of compressing the data to a small number of bits so that a data structure $f(\cdot)$ with access only to these small number of bits satisfies 
$$|f(\beta) - \Lcal(\beta)| \leq \epsilon \Lcal(\beta),$$ 
for all $\beta \in \R^d$. Recall that the running time of the computation compressing the data to a small number of bits and evaluating $f(\beta)$ for a given query $\beta$ is unbounded in this model. Hence, Theorem~\ref{thm:space_lower_bound} provides a strong lower bound on the space needed for \textit{any} compression of the data that can be used to compute an $\epsilon$-relative error approximation to the logistic loss, including, but not limited to, coresets. 

At a high level, our proof operates by showing that a relative error approximation to logistic loss can be used to obtain a relative error approximation to ReLu regression, which in turn can be used to construct a relative error $\ell_1$-subspace embedding. Previously,~\citet{li2021tight} lower bounded the worst case space complexity of any data structure that maintains an $\ell_1$-subspace embedding by reducing the problem to the well-known~\texttt{INDEX} problem in communication complexity. See Appendix~\ref{sxn:proof_appendix} for the proof of Theorem~\ref{thm:space_lower_bound}.
\begin{theorem}\label{thm:space_lower_bound}
    Let $\Xb \in \R^{n \times d}$ and $\yb \in \{-1, 1\}^n$ be the data matrix and label vector such that each entry of $\Xb$ is specified by $\Ocal(\log nd)$ bits. If $f(\cdot)$ is a data structure such that for all $\beta \in \R^d$, $|f(\beta) - \Lcal(\beta)| \leq \epsilon\Lcal(\beta),$ then $f(\cdot)$ requires $\Omegatil\left(\frac{d}{\epsilon^2}\right)$ space in the worst case, provided that $d = \Omega(\log 1/\epsilon)$ and $n = \Omegatil\left(d\epsilon^{-2}\right)$ and $\epsilon > 0$ is sufficiently small. 
\end{theorem}
Prior work on coreset construction for logistic regression critically depends on the data complexity measure $\mu_\yb(\Xb)$, which was first introduced in \cite{munteanu2018coresets}, and is defined as follows.
\begin{definition}\label{def:complexity_measure} (Classification Complexity Measure \cite{munteanu2018coresets}) 
    For any $\Xb \in \R^{n \times d}$ and $\yb \in \{-1, 1\}^n$, let $\mu_\yb(\Xb) = \sup_{\beta \neq \zero} \frac{\|(\Db_\yb\Xb\beta)^+\|_1}{\|(\Db_\yb\Xb\beta)^-\|_1}$, where $\Db_\yb$ is a diagonal matrix with $\yb$ as its diagonal, and $(\Db_\yb\Xb\beta)^+$ and $(\Db_\yb\Xb\beta)^-$ denote the positive and the negative entries of $\Db_\yb\Xb\beta$ respectively.
\end{definition}
The work of~\citet{mai2021coresets} showed that sampling $\Ocal\left(\frac{d \cdot \mu_\yb(\Xb)^2}{\epsilon^2}\right)$ rows of $\Xb$ yields an $\epsilon$-relative error coreset for logistic loss with high probability, where $\mu_\yb(\Xb)$ is the data complexity measure of Definition~\ref{def:complexity_measure}. This leaves a gap of $\Ocaltil(d\cdot\mu_\yb(\Xb)^2)$ between our space complexity lower bound and the best known upper bound for the space needed to approximate logistic loss to relative error. Closing this gap is an open problem for future work.

Understanding whether the complexity measure of Definition~\ref{def:complexity_measure} truly explains the performance of coresets on real-world data sets would help guide further improvements to coreset construction for logistic regression. However, prior work conjectured that $\mu_\yb(\Xb)$ was hard to compute, while providing a polynomial time algorithm to approximate the measure to within a $\operatorname{poly}(d)$-factor (see Theorem 3 of \citet{munteanu2018coresets}). We refute this conjecture by showing that the complexity measure $\mu_\yb(\Xb)$ can in fact be computed efficiently via linear programming.
%
\begin{theorem}\label{thm:compute_complexity_measure}
    If the complexity measure $\mu_\yb(\Xb)$ of Definition~\ref{def:complexity_measure} is finite, it can be computed \textit{exactly} by solving a linear program with $2n$ variables and $4n$ constraints.
\end{theorem}
\begin{proof}
    We now derive a linear programming formulation formulation to compute the complexity measure in Definition \ref{def:complexity_measure}. Note we flip the numerator and denominator from Definition \ref{def:complexity_measure} without loss of generality. Define $\beta^* \in \R^d$ as, ($C$ is an arbitrary positive constant):
\begin{flalign*}
    \beta^* &= \argmax_{\beta \neq \zero} \frac{\|(\Db_\yb\Xb\beta)^-\|_1}{\|(\Db_\yb\Xb\beta)^+\|_1} \\
    &= \underset{\beta \in \R^d}{\argmax} ~ \|(\Db_\yb\Xb\beta)^-\|_1 \text{ such that } \|(\Db_\yb\Xb\beta)^+\|_1 \leq C \\
    &= \underset{\beta \in \R^d}{\argmax} ~ \|\Db_\yb\Xb\beta\|_1 \text{ such that } \|(\Db_\yb\Xb\beta)^+\|_1 \leq C.
\end{flalign*}

Now we reformulate the last constraint: 
\begin{align*}
\|(\Db_\yb\Xb\beta)^+\|_1 &= \sum_{i=1}^n \max\{[\Db_\yb\Xb\beta]_i, 0\} \\
&= \sum_{i=1}^n \frac{1}{2} [\Db_\yb\Xb\beta]_i + \frac{1}{2}|[\Db_\yb\Xb\beta]_i| \\
&= \frac{1}{2}\one^T\Db_\yb\Xb\beta + \frac{1}{2}\|\Db_\yb\Xb\beta\|_1.  
\end{align*}
Therefore, the above formulation is equivalent to:
\begin{flalign*}
    \beta^* = \underset{\beta \in \R^d}{\argmax}& ~ \|\Db_\yb\Xb\beta\|_1 \\
    \text{ such that }& \frac{1}{2}\one^T\Db_\yb\Xb\beta + \frac{1}{2}\|\Db_\yb\Xb\beta\|_1 \leq C \\ 
    = \underset{\beta \in \R^d}{\argmin}& ~ \one^T\Db_\yb\Xb\beta \text{ such that } \|\Db_\yb\Xb\beta\|_1 \leq C.
\end{flalign*}
Next, we replace $\Db_\yb\Xb\beta$ with a single vector $\zb \in \R^n$ and a linear constraint to guarantee that $\zb \in \operatorname{Range}(\Db_\yb\Xb)$. Let $\Pb_R \in \R^{n \times n}$ be the orthogonal projection to $\operatorname{Range}(\Db_\yb\Xb)$.
\begin{gather*}
    \zb^* = \underset{\beta \in \R^d}{\argmin} ~ \one^T\zb \\\text{ such that } \|\zb\|_1 \leq C, ~(\Ib - \Pb_R)\zb = \zero.
\end{gather*}
Next, we solve this formulation by constructing a linear program such that $[\zb_+, \zb_-] \in \R^{2n}$ corresponds to the absolute value of the positive and negative elements of $\zb$:  \begin{align*}
    \zb^* = \underset{\beta \in \R^d}{\argmin}& ~ \one_{n}^T(\zb_+ -\zb_-) \\ 
    \text{ such that }& \sum_{i=1}^{2n} [\zb_+, \zb_-]_i \leq C,\\ 
    &(\Ib - \Pb_R)(\zb_+ - \zb_-)  = \zero,\ \ 
    \zb_+,\zb_- \geq 0.
\end{align*}
Observe that this is a linear program with $2n$ variables and $4n$ constraints.  After solving this program for $\zb_+^*$ and $\zb_-^*$, we can compute $\zb^* = \zb_+^* - \zb_-^*$.  From this, we can compute $\beta^*$ by solving the linear system $\zb^* = \Db_\yb\Xb\beta^*$, which is guaranteed to have a solution by the linear constraint $(\Ib - \Pb_R)\zb^* = \zero$.

After solving for $\beta^*$, we can compute (recall that the numerator and denominator are flipped  readability without loss of generality) $$\mu_\yb(\Xb) = \frac{\|(\Db_\yb\Xb\beta^*)^-\|_1}{\|(\Db_\yb\Xb\beta^*)^+\|_1},$$ completing the proof.
\end{proof}
Prior experimental evaluation~\cite{munteanu2018coresets, mai2021coresets} of coreset constructions relied on estimates of $\mu_\yb(\Xb)$ using the method provided by \citet{munteanu2018coresets}. Given the large $\operatorname{poly}(d)$ approximation factor of this algorithm, it is unclear how reliable these estimates are. Our linear programming formulation provides a novel way to relate the proposed complexity measure and coreset performance.

\subsection{Low-rank approximations in logistic regression}

In this section, we consider the performance of low-rank approximations to the matrix $\Xb$ in order to approximately compute the logistic loss function. Low-rank approximations are often used to reduce the time complexity or to denoise and improve the performance of clustering and classification algorithms (such as logistic regression) by replacing a matrix with its low rank factors~\cite{Drineas2004,Paul2014}. To the best of our knowledge, its use for logistic regression has not been explicitly analyzed in prior work. 

Using low-rank approximations of $\Xb$ to estimate the logistic loss is appealing due to the extensive body work on fast constructions of low-rank approximations via sketching, sampling, and direct methods \cite{kishore2017literature}. We show that a spectral approximation provides an additive error guarantee for the logistic loss and that this guarantee is tight on worst-case inputs.

\begin{theorem}\label{thm:logistic_loss_upper}
    If $\Xb, \Xbtil \in \R^{n \times d}$, then for all $\beta \in \R^d$,
    \begin{gather*}
        |\Lcal(\beta; \Xb) - \Lcal(\beta, \Xbtil)| \leq \sqrt{n} \|\Xb - \Xbtil\|_2 \|\beta\|_2.
    \end{gather*}
\end{theorem}
We note that Theorem~\ref{thm:logistic_loss_upper} holds for any matrix $\Xbtil \in \R^{n \times d}$ that approximates $\Xb$ with respect to the spectral norm, and does not necessitate that $\Xbtil$ has low-rank. We now provide a matching lower-bound for the logistic loss function in the same setting.
\begin{theorem}\label{thm:logistic_loss_lower}
    For every $d,n \in \mathbb{N}$ where $d \geq n$, there exists a data matrix $\Xb \in \R^{n \times d}$, label vector $\yb \in \{-1, 1\}^n$, parameter vector $\beta \in \R^d$, and spectral approximation $\Xbtil \in \R^{n \times d}$ such that:
    \begin{gather*}
        |\Lcal(\beta; \Xbbar) - \Lcal(\beta; \Xb)| \geq (1-\delta)\sqrt{n}\|\Xb - \Xbbar\|_2\|\beta\|_2,
    \end{gather*}
    for every $\delta > 0$.  
    Hence, the guarantee of Theorem \ref{thm:logistic_loss_upper} is tight in the worst case.
\end{theorem}

\section{Bounding the forward error}\label{sxn:forward_error}

In this section, we provide upper and lower bounds for the forward error when running logistic regression on the lower-dimensional data matrix $\Xb\Pb_k \in \R^{n \times k}$, with $k \ll d$, where $d$ is the dimensionality of the original data points (see Section~\ref{section:contributions} for notations). Recall that $\Pb_k \in \R^{d \times k}$ denotes the linear dimensionality reduction matrix, whose columns are pairwise orthogonal and normal ($\Pb_k^T \Pb_k = \Ib_k$). Using our notation, the solution to the original logistic regression problem is given by eqn.~(\ref{eqn:blrr}), while the solution to the dimensionally reduced logistic regression problem is as stated in eqn.~(\ref{eqn:sketchedregression1}).
Notice that at this point we assume that the regularization parameters for the two problems are not equal; we will see in our proof of Theorem~\ref{thm:forward_error_upper_bound} that setting $\mu = \lambda$ simplifies our bound without increasing the forward error, at least if no additional assumptions are made. 

\subsection{Upper bound}\label{sxn:forward_upper}
We now prove the following Theorem that upper bounds the forward error. We will use the notation introduced in Section~\ref{section:contributions}.
\begin{theorem}\label{thm:forward_error_upper_bound}
    Let the regularization terms of the dimensionally-reduced (eqn.~(\ref{eqn:sketchedregression1})) and the original (eqn.~(\ref{eqn:blrr})) logistic regression problems be equal, i.e., $\mu = \lambda$. Then, the forward error of the dimensionally-reduced logistic regression problem satisfies the following inequality: 
    \begin{align*}
        \|\Pb_k\beta_k - \beta_d\|_2^2
        &\leq \frac{2}{\lambda}\Phi(\Xb,\yb,\Pb_k),
    \end{align*}
    where $\Phi(\Xb,\yb,\Pb_k) = \beta_d^T(\Ib - \Pb_k\Pb_k^T) \Xb^T\Db_\yb \wb(\Pb_k\beta_k)$.
\end{theorem}
\begin{proof}
We start by computing the gradients of the loss functions for the original and the dimensionally-reduced problems. Recall that the loss function of eqn.~(\ref{eqn:blrr}) is a function over vectors $\beta \in \R^d$, while the loss function of eqn.~(\ref{eqn:sketchedregression1}) is over vectors $\beta \in \R^k$. For notational convenience, let $\Xbbar = -\Db_\yb\Xb$ and let $\wbbar(\beta) \in \R^n$ have entries $\wbbar(\beta)_i = \sigma(\xbbar_i^T\beta)$ (we do note that $\wb(\beta) = \wbbar(\beta)$). Using our notation, the gradients are:
\begin{gather}
    \nabla \Lcal_d(\beta) = \Xbbar^T \wbbar(\beta) + \lambda \cdot \beta
    \quad\text{and}\nonumber\\
    \nabla \Lcal_k(\beta) = \Pb_k^T\Xbbar^T \wbbar(\Pb_k\beta) + \mu \cdot \beta.\label{eqn:grall}
\end{gather}
After solving the dimensionally-reduced problem, we map its $k$-dimensional solution $\beta_k$ to the original $d$-dimensional domain by premultiplying by the dimensionality reduction matrix $\Pb_k$ to compute $\Pb_k\beta_k \in \R^d$. Since $\beta_k$ is the optimal solution to the sketched problem, we know that $\nabla \Lcal_k(\beta_k)$ is equal to zero and thus $\Pb_k\nabla \Lcal_k(\beta_k) = 0$.  This allows us to simplify the gradient of the original loss function $ \nabla\Lcal_d$ computed at $\Pb_k\beta_k$ as follows:
\begin{align}
    \nabla \Lcal_d(\Pb_k\beta_k)
    &= \Xbbar^T \wbbar(\Pb_k\beta_k) + \lambda \cdot \Pb_k\beta_k \\
    &= \Xbbar^T \wbbar(\Pb_k\beta_k) + \lambda \cdot \Pb_k\beta_k - \Pb_k\nabla \Lcal_k(\beta_k)\\
    &= \Xbbar^T \wbbar(\Pb_k\beta_k) + \lambda \cdot \Pb_k\beta_k \\
    &\quad\quad- \Pb_k[\Pb_k^T\Xbbar^T \wbbar(\Pb_k\beta_k) + \mu \cdot \beta_k] \nonumber\\
    &= (\Ib - \Pb_k\Pb_k^T) \Xbbar^T \wbbar(\Pb_k\beta_k) + (\lambda - \mu) \Pb_k \beta_k.\label{eqn:pd1}
\end{align} 
The above derivations are immediate using the gradients of the loss functions from eqn.~(\ref{eqn:grall}). Let 
\begin{align}
\vb = \frac{\Pb_k\beta_k - \beta_d}{\|\Pb_k\beta_k - \beta_d\|_2}.\label{eqn:pdv}
\end{align}
In words, $\vb$ is a unit vector pointing from $\beta_d$ to $\Pb_k\beta_k$, which allows us to bound the forward error $\|\Pb_k\beta_k - \beta_d\|_2$ using values $t>0$ that satisfy: 
\begin{gather}
    \Lcal_d(\Pb_k\beta_k - t \cdot \vb) - \Lcal_d(\Pb_k\beta_k) \geq t\cdot (-\vb^T\nabla \Lcal_d(\Pb_k\beta_k)) + \frac{t^2 \cdot \lambda}{2}.\label{eqn:vv1}
\end{gather}
In order to understand the above equation, note that $\Lcal_d(\beta)$ is $\lambda$-strongly convex.  Therefore, we can \textit{lower}-bound the loss at any point by using a truncated Taylor series and setting the Hessian term to $\lambda \cdot \Ib$~(see \cite{nesterov1998introductory} for definitions and properties of strong convexity).  As we increase $t$ in the truncated Taylor series, the point $\Pb_k\beta_k - t\cdot \vb$ will move from $\Pb_k\beta_k$ to $\beta_d$ and then past it.  If $\Lcal_d(\Pb_k\beta_k - t \cdot \vb) \geq \Lcal_d(\Pb_k\beta_k)$, then we know we have passed the point $\beta_d$ (the optimal solution to the original logistic regression problem), since $\beta_d$ minimizes the loss function $\Lcal_d(\cdot)$.  We can now derive our bound by first setting the left-hand side of eqn.~(\ref{eqn:vv1}) to zero and solving for $t^* >0$:
\begin{gather*}
    0 = -\vb^T\nabla \Lcal_d(\Pb_k\beta_k) + \nicefrac{t^* \cdot \lambda}{2} 
    \Rightarrow
     \nicefrac{t^* \cdot \lambda}{2} =  \vb^T\nabla \Lcal_d(\Pb_k\beta_k) \\
    \Rightarrow t^* = \nicefrac{2}{\lambda}\vb^T(\Ib - \Pb_k\Pb_k^T) \Xbbar^T \wbbar(\Pb_k\beta_k)
    + \nicefrac{2(\lambda - \mu)}{\lambda} \vb^T \Pb_k \beta_k.
\end{gather*}
In the last derivation we used eqn.~(\ref{eqn:pd1}). We now proceed by expanding $\vb$ using eqn.~(\ref{eqn:pdv}) and 
\begin{align*}
    t^*\|\Pb_k\beta_k - \beta_d\|_2 &= \frac{2}{\lambda}(\Pb_k\beta_k - \beta_d)^T(\Ib - \Pb_k\Pb_k^T) \Xbbar^T \wbbar(\Pb_k\beta_k) + \frac{2(\lambda - \mu)}{\lambda} (\Pb_k\beta_k - \beta_d)^T \Pb_k \beta_k \\
    &= \frac{-2}{\lambda}\beta_d^T(\Ib - \Pb_k\Pb_k^T) \Xbbar^T \wbbar(\Pb_k\beta_k) + \frac{2(\lambda - \mu)}{\lambda} (\|\Pb_k\beta_k\|_2^2 - \beta_d^T \Pb_k \beta_k).
\end{align*}
To derive the second equality, we used the fact that $\beta_k^T \Pb_k^T(\Ib - \Pb_k\Pb_k^T)$ is equal to zero. We now conclude the proof by using $t^*\geq\|\Pb_k\beta_k - \beta_d\|_2$, which must hold since the loss of $\beta$ decreases monotonically as $\beta$ goes from $\Pb_k\beta_k$ to the optimum $\beta_d$.
\begin{align*}
    \|\Pb_k\beta_k - \beta_d\|_2^2 
    &\leq \frac{-2}{\lambda}\beta_d^T(\Ib - \Pb_k\Pb_k^T) \Xbbar^T \wbbar(\Pb_k\beta_k) + \frac{2(\lambda - \mu)}{\lambda} (\|\beta_k\|_2^2 - \beta_d^T \Pb_k \beta_k) \\
    &\leq \frac{2}{\lambda}\beta_d^T(\Ib - \Pb_k\Pb_k^T) \Xb^T\Db_\yb \wb(\Pb_k\beta_k) + \frac{2(\lambda - \mu)}{\lambda} (\|\beta_k\|_2^2 - \beta_d^T \Pb_k \beta_k).
\end{align*}
In the above derivation we also used the fact that $\|\Pb_k \beta_k\|_2 = \|\beta_k\|_2$, since $\Pb_k$ is a column-orthogonal matrix. We can now conclude the proof by setting $\lambda = \mu$, which cancels the second term in the right-hand side of the above inequality. We do note that this term scales as a function of the difference between the norm (squared) of the solution to the dimensionally-reduced problem $\beta_k$ and the inner product between the solution to the original problem $\beta_d$ and the ``mapped'' solution $\Pb_k \beta_k$. One could envision settings where this difference is sufficiently large so that using a value of $\mu$ that is not equal to $\lambda$ and makes the second term negative improves the overall bound. This seems challenging without additional assumptions on the behavior of the aforementioned difference, which in turn depends on our choice of $\mu$. Therefore, to simplify our bound and avoid additional complicated assumptions, we set $\mu$ equal to $\lambda$.
\end{proof}

\subsection{A matching lower bound}\label{sxn:forward_lower}

We now proceed to prove a matching (up to constants and $\lambda$-factors) lower bound for Theorem~\ref{thm:forward_error_upper_bound}. Towards that end, we start by stating two lemmas from prior work that will be critical in our proofs. The first lemma shows that upper bounding the $\ell_2$-norm of the Hessian also upper bounds the norm of the gradient of a function, when these quantities exist.

\vspace{0.02in}\begin{lemma}\label{lemma:lipschitz_gradient}
    \citep[Lemma 1.2.2]{nesterov1998introductory} A function $f(x)$ is twice differentiable on $\R^n$ and has an $L$-Lipschitz gradient if and only if
    $\|\nabla^2f(\xb)\|_2 \leq L$, for all $\xb \in \R^n$.
\end{lemma}
The second lemma shows that if the gradient of a function is Lipschitz, then the function can be upper bounded by a quadratic function.
\vspace{0.02in}\begin{lemma}\label{lemma:lipschitz_convexity}
    \citep[Lemma 1.2.3]{nesterov1998introductory} If a function $f(x)$ is differentiable on $\R^n$ and has an $L$-Lipschitz gradient, then for all $\xb,\yb \in \R^n$,
    \begin{gather*}
        |f(\yb) - f(\xb) - (\yb - \xb)^T \nabla f(\xb)| \leq \nicefrac{L}{2} \|\yb - \xb\|_2^2.
    \end{gather*}
\end{lemma}
We are now ready to present our lower-bound proof. 
\begin{theorem}\label{thm:forward_error_lower_bound}
    Let the regularization terms of the dimensionally-reduced (eqn.~(\ref{eqn:sketchedregression1})) and the original (eqn.~(\ref{eqn:blrr})) logistic regression problems be equal, i.e., $\mu = \lambda$. Then, the forward error of the dimensionally-reduced logistic regression problem satisfies the following inequality: 
    \begin{gather*}
         \|\Pb_k\beta_k - \beta_d\|_2^2 \geq \frac{4}{ \|\Xb\|_2^2 + \lambda} \cdot\Phi(\Xb,\yb,\Pb_k),
    \end{gather*}
    where $\Phi(\Xb,\yb,\Pb_k) = \beta_d^T(\Ib - \Pb_k\Pb_k^T) \Xb^T\Db_\yb \wb(\Pb_k\beta_k)$.
\end{theorem}

\begin{proof}
    First, define the diagonal matrix $\Db(\beta) \in \R^{n\times n}$, whose diagonal entries $\Db(\beta)_{ii}$ are equal to $\sigma(\xb_i^T\beta)\sigma(-\xb_i^T\beta)$ for all $i=1\ldots n$; recall from the previous section that $\sigma(\cdot)$ is the logistic function. The Hessian of the regularized loss function of the original problem (see eqn.~(\ref{eqn:blrr}) is:
$
    \nabla^2 \Lcal_d(\beta) = \Xbbar^T \Db(\beta) \Xbbar + \lambda \cdot \Ib.
$
By the definition of $\sigma(\cdot)$ it follows that $\sigma(\xb_i^T\beta)\sigma(-\xb_i^T\beta) \in (0,1/4]$ and therefore $\|\Db(\beta)\|_2 \leq 1/4$. We can now upper-bound the spectral norm of the Hessian using sub-additivity and sub-multiplicativity of the spectral norm as follows:
\begin{align*}
    \|\nabla^2 \Lcal_d(\beta)\|_2 & = \|\Xbbar^T \Db(\beta) \Xbbar + \lambda \cdot \Ib\|_2 \\
    &\leq \|\Xbbar^T\|_2 \cdot \|\Db(\beta)\|_2 \cdot \|\Xbbar\|_2 + \lambda\cdot \|\Ib\|_2 \\
    &= \nicefrac{1}{4} \cdot \|\Xbbar\|_2^2 + \lambda.
\end{align*}

This shows that $\Lcal_d(\beta)$ is $\alpha$-smooth with $\alpha = \nicefrac{1}{4} \cdot \|\Xbbar\|_2^2 + \lambda$.  This implies two facts: first, using Lemma~\ref{lemma:lipschitz_gradient}, the gradient $\nabla \Lcal_d(\beta)$ is $\alpha$-Lipschitz with respect to the $\ell_2$-norm. 
Second, we can upper-bound $\Lcal_d(\beta)$ using  a quadratic function as follows:
\begin{gather*}\label{eq:intermediate_upper}
    \Lcal_d(\Pb_k\beta_k - t \cdot \vb) - \Lcal_d(\Pb_k\beta_k)
    \leq t\cdot (-\vb^T\nabla \Lcal_d(\Pb_k\beta_k)) + \frac{t^2 \cdot \alpha}{2}, \quad
\text{for all }t \in [0,1].
\end{gather*}
The above follows from Lemma~\ref{lemma:lipschitz_convexity}; 
$\vb$ is the same vector as in the previous section (see eqn.~(\ref{eqn:pdv})). Notice that both the left and the right hand side of the above inequality are functions of the variable $t$. We let $f(t) = \Lcal_d(\Pb_k\beta_k - t \cdot \vb) - \Lcal_d(\Pb_k\beta_k)$ denote the left-hand side of the inequality and we let $q(t) = t\cdot -\vb^T\nabla \Lcal_d(\Pb_k\beta_k) + \nicefrac{t^2 \cdot \alpha}{2}$ denote the quadratic function at the right-hand side of the inequality. We know that $f(t)$ is minimized when $\Pb_k\beta_k - t \cdot \vb = \beta_d$, since $\beta_d$ minimizes $\Lcal_d(\beta)$. It immediately follows that $f(t)$ is minimized by setting $t = \|\Pb_k\beta_k - \beta_d\|_2$.  We can therefore lower-bound the forward error by first showing that 
\begin{align}
\argmin_{t \geq 0} f(t) \geq \argmin_{t \geq 0} q(t).\label{eqn:pd3}
\end{align}
Towards that end, we express the derivatives of $f(t)$ and $q(t)$ as functions of the gradient of $\Lcal_d(\beta)$ as follows:
\begin{gather*}
    f'(t)  = -\vb^T \nabla \Lcal_d(\Pb_k\beta_k - t \cdot \vb), \quad \text{and} \quad \\
    q'(t) = -\vb^T\nabla \Lcal_d(\Pb_k\beta_k) + t \cdot \alpha.
\end{gather*}
Since $\nabla \Lcal_d(\beta)$ is $\alpha$-Lipschitz, $\| \nabla \Lcal_d(\Pb_k\beta_k - t \cdot \vb) -  \nabla \Lcal_d(\Pb_k\beta_k)\|_2 \leq \alpha \cdot t$. Using the fact that $\|\vb\|_2=1$, we get
\begin{align}
\vb^T\nabla &\Lcal_d(\Pb_k\beta_k - t \cdot \vb) - \vb^T\nabla \Lcal_d(\Pb_k\beta_k) \nonumber \\
&\leq \|\vb\|_2
\| \nabla \Lcal_d(\Pb_k\beta_k - t \cdot \vb) -  \nabla \Lcal_d(\Pb_k\beta_k)\|_2 \nonumber \\
&\leq \alpha \cdot t. \label{eqn:pd2}
\end{align}
We are now ready to show that the derivative of $q(t)$ upper-bounds the derivative of $f(t)$:
\begin{align*}
     q'(t) - f'(t)
     &= -\vb^T\nabla \Lcal_d(\Pb_k\beta_k) + t \cdot \alpha  + \vb^T \nabla \Lcal_d(\Pb_k\beta_k - t \cdot \vb) \\
     &= t \cdot \alpha - (\vb^T\nabla \Lcal_d(\Pb_k\beta_k - t \cdot \vb) - \vb^T\nabla \Lcal_d(\Pb_k\beta_k)) \\
     &\geq t\cdot \alpha - t \cdot \alpha = 0.
\end{align*}
The inequality follows from eqn.~(\ref{eqn:pd2}). Since initially $f'(0)$ and $q'(0)$ are both negative or zero, it must be the case that $q'(t) = 0$ first, hence $\argmin_{t>0} f(t) \geq \argmin_{t\geq 0} q(t)$, thus proving eqn.~(\ref{eqn:pd3}). Recall that the minimizer for $f(t)$ is $\argmin_{t>0} f(t) = \|\Pb_k\beta_k - \beta_d\|_2$, while we can directly solve for the minimizer for $q(t)$ to get $$\argmin_{t>0} q(t) = \frac{1}{\alpha} \vb^T\nabla \Lcal_d(\Pb_k\beta_k).$$ This gives the following bound on the forward error: 
\begin{gather*}
    \frac{1}{\alpha} \vb^T\nabla \Lcal_d(\Pb_k\beta_k)
    \leq \|\Pb_k\beta_k - \beta_d\|_2.
\end{gather*}
We can manipulate this lower bound so that it matches the upper bound of Theorem~\ref{thm:forward_error_upper_bound} up to constant factors. Using the definition of $\vb$ from eqn.~(\ref{eqn:pdv}) we get:
\begin{align*}
    \|\Pb_k\beta_k - \beta_d\|_2 &\geq \frac{1}{\alpha} \vb^T \nabla \Lcal_d(\Pb_k\beta_k) \\
    &=
    \frac{1}{\alpha \|\Pb_k\beta_k - \beta_d\|_2 }
    (\Pb_k\beta_k - \beta_d)^T\nabla \Lcal_d(\Pb_k\beta_k).
\end{align*}
Recall that, for $\mu = \lambda$, $\nabla \Lcal_d(\Pb_k\beta_k) = (\Ib - \Pb_k\Pb_k^T) \Xbbar^T \wbbar(\Pb_k\beta_k)$ and rewrite the above inequality to get:
\begin{align*}
    \|\Pb_k\beta_k& - \beta_d\|_2^2 \geq \frac{1}{\alpha} (\Pb_k\beta_k - \beta_d)^T\nabla \Lcal_d(\Pb_k\beta_k) \\
    &= \frac{1}{\alpha} (\Pb_k\beta_k - \beta_d)^T (\Ib - \Pb_k\Pb_k^T) \Xbbar^T  \wbbar(\Pb_k\beta_k)\\
    &=\frac{1}{\alpha} \cdot (- \beta_d^T (\Ib - \Pb_k\Pb_k^T) \Xbbar^T  \wbbar(\Pb_k\beta_k)) \\
    &=\frac{1}{\alpha} \cdot (\beta_d^T (\Ib - \Pb_k\Pb_k^T) \Xb^T\Db_\yb  \wb(\Pb_k\beta_k)).
\end{align*}
The second-to-last equality follows since $\beta_k^T\Pb_k(\Ib - \Pb_k\Pb_k^T)$ is equal to zero.  Finally, using $\|\Xbbar\|_2 = \|\Xb\|_2$ to conclude the theorem statement.
\end{proof}
We note that the bound of Theorem~\ref{thm:forward_error_lower_bound} is useful only when $\|\Xb\|_2^2 \geq \lambda$. This is actually the only interesting setting, since if $\|\Xb\|_2^2 < \lambda$ then the solution to both the sketched and the original problem is trivial and equal to zero.

Finally, it immediately follows from Theorem~\ref{thm:forward_error_lower_bound} that if the two norm of the input matrix $\Xb$ is constant (i.e., $\|\Xb\|_2=O(1)$), our upper and lower bounds differ by a constant factor that only depends on $\lambda$:
\begin{align*}
\nicefrac{2}{\lambda} \cdot \alpha = \nicefrac{1}{2\lambda} \cdot \|\Xbbar\|_2^2 + 2 = \nicefrac{1}{2\lambda} \cdot \|\Xb\|_2^2 + 2 = O(\nicefrac{1}{\lambda}).\label{eqn:lowerbound}
\end{align*}

\section{Future work}

The most pressing open problem for future work is to explore theoretical connections between our bounds and in- and out-of-sample classification accuracy for logistic regression and GLMs. Another future direction could investigate the existence of similar bounds for non-linear dimensionality reduction techniques for logistic regression and GLMs.

\bibliographystyle{plainnat}
\begin{small}
\bibliography{bibliography}

\begin{thebibliography}{42}
\providecommand{\natexlab}[1]{#1}
\providecommand{\url}[1]{\texttt{#1}}
\expandafter\ifx\csname urlstyle\endcsname\relax
  \providecommand{\doi}[1]{doi: #1}\else
  \providecommand{\doi}{doi: \begingroup \urlstyle{rm}\Url}\fi

\bibitem[Avron et~al.(2010)Avron, Maymounkov, and Toledo]{avron2010blendenpik}
Haim Avron, Petar Maymounkov, and Sivan Toledo.
\newblock {Blendenpik: Supercharging LAPACK's least-squares solver}.
\newblock \emph{SIAM Journal on Scientific Computing}, 32\penalty0
  (3):\penalty0 1217--1236, 2010.

\bibitem[Avron et~al.(2016)Avron, Clarkson, and Woodruff]{Avron2016}
Haim Avron, Kenneth~L. Clarkson, and David~P. Woodruff.
\newblock {Sharper Bounds for Regression and Low-Rank Approximation with
  Regularization}.
\newblock 2016.
\newblock URL \url{http://arxiv.org/abs/1611.03225}.

\bibitem[Chowdhury et~al.(2018)Chowdhury, Yang, and
  Drineas]{chowdhury2018iterative}
Agniva Chowdhury, Jiasen Yang, and Petros Drineas.
\newblock An iterative, sketching-based framework for ridge regression.
\newblock In \emph{International Conference on Machine Learning}, pages
  989--998. PMLR, 2018.

\bibitem[Clarkson and Woodruff(2009)]{clarkson2009numerical}
Kenneth~L Clarkson and David~P Woodruff.
\newblock Numerical linear algebra in the streaming model.
\newblock In \emph{Proceedings of the forty-first Annual ACM Symposium on
  Theory of Computing}, pages 205--214, 2009.

\bibitem[Clarkson and Woodruff(2013)]{Clarkson2013a}
Kenneth~L. Clarkson and David~P. Woodruff.
\newblock {Low rank approximation and regression in input sparsity time}.
\newblock In \emph{Proceedings of the 45th annual ACM Symposium on Theory of
  Computing (STOC)}, 2013.

\bibitem[Clarkson et~al.(2013)Clarkson, Drineas, Magdon-Ismail, Mahoney, Meng,
  and Woodruff]{Clarkson2013b}
Kenneth~L. Clarkson, Petros Drineas, Malik Magdon-Ismail, Michael~W. Mahoney,
  Xiangrui Meng, and David~P. Woodruff.
\newblock {The Fast Cauchy Transform and Faster Robust Linear Regression.}
\newblock \emph{Proceedings of the 24th Annual ACM-SIAM Symposium on Discrete
  Algorithms (SODA)}, 45\penalty0 (3):\penalty0 466--477, 2013.

\bibitem[Clarkson et~al.(2016)Clarkson, Drineas, Magdon-Ismail, Mahoney, Meng,
  and Woodruff]{Clarkson2016}
K.L. Clarkson, P.~Drineas, M.~Magdon-Ismail, M.W. Mahoney, X.~Meng, and D.P.
  Woodruff.
\newblock {The fast Cauchy transform and faster robust linear regression}.
\newblock \emph{SIAM Journal on Computing}, 45\penalty0 (3), 2016.

\bibitem[Cover(1999)]{cover1999elements}
Thomas~M Cover.
\newblock \emph{Elements of information theory}.
\newblock John Wiley \& Sons, 1999.

\bibitem[Cramer(2003)]{Cramer2003}
J.S. Cramer.
\newblock {The Origins of Logistic Regression}.
\newblock \emph{SSRN Electronic Journal}, 2003.

\bibitem[Derezi{\'{n}}ski and Mahoney(2021)]{Derezinski2021}
Micha{\l} Derezi{\'{n}}ski and Michael Mahoney.
\newblock {Determinantal Point Processes in Randomized Numerical Linear
  Algebra}.
\newblock \emph{Notices of the American Mathematical Society}, 68, 2021.

\bibitem[Derezinski and Warmuth(2018)]{Derezinski2017}
Michal Derezinski and Manfred~K. Warmuth.
\newblock Subsampling for ridge regression via regularized volume sampling.
\newblock In \emph{International Conference on Artificial Intelligence and
  Statistics}, volume~84, pages 716--725, 2018.

\bibitem[Drineas and Mahoney(2018)]{Drineas2017}
P.~Drineas and M.W. Mahoney.
\newblock {Lectures on Randomized Numerical Linear Algebra}, 2018.

\bibitem[Drineas et~al.(2018)Drineas, Ipsen, Kontopoulou, and
  Magdon-Ismail]{Drineas2018}
P.~Drineas, I.C.F. Ipsen, E.-M. Kontopoulou, and M.~Magdon-Ismail.
\newblock {Structural convergence results for approximation of dominant
  subspaces from block Krylov spaces}.
\newblock \emph{SIAM Journal on Matrix Analysis and Applications}, 39\penalty0
  (2), 2018.

\bibitem[Drineas and Mahoney(2016)]{Drineas2016}
Petros Drineas and M.W. Michael~W. Mahoney.
\newblock {RandNLA: Randomized Numerical Linear Algebra}.
\newblock \emph{Communications of the ACM}, 59\penalty0 (6):\penalty0 80--90,
  2016.

\bibitem[Drineas et~al.(2004)Drineas, Frieze, Kannan, Vempala, and
  Vinay]{Drineas2004}
Petros Drineas, Alan Frieze, Ravi Kannan, Santosh Vempala, and V.~Vinay.
\newblock {Clustering Large Graphs via the Singular Value Decomposition}.
\newblock \emph{Machine Learning}, 56\penalty0 (1-3):\penalty0 9--33, 2004.

\bibitem[Drineas et~al.(2006)Drineas, Mahoney, and Muthukrishnan]{DMM06}
Petros Drineas, Michael~W. Mahoney, and S.~Muthukrishnan.
\newblock {Sampling Algorithms for $\ell_2$ Regression and Applications}.
\newblock In \emph{Proceedings of the 17th Annual ACM-SIAM Symposium on
  Discrete Algorithms (SODA)}, pages 1127--1136, 2006.

\bibitem[Drineas et~al.(2011)Drineas, Mahoney, Muthukrishnan, and
  Sarl{\'{o}}s]{Drineas2011}
Petros Drineas, Michael~W. Mahoney, S.~Muthukrishnan, and Tam{\'{a}}s
  Sarl{\'{o}}s.
\newblock {Faster least squares approximation}.
\newblock \emph{Numerische Mathematik}, 117\penalty0 (2):\penalty0 219--249,
  2011.

\bibitem[Elenberg et~al.(2018)Elenberg, Khanna, Dimakis, and
  Negahban]{elenberg2018restricted}
Ethan~R Elenberg, Rajiv Khanna, Alexandros~G Dimakis, and Sahand Negahban.
\newblock Restricted strong convexity implies weak submodularity.
\newblock \emph{The Annals of Statistics}, 46\penalty0 (6B):\penalty0
  3539--3568, 2018.

\bibitem[Halko et~al.(2011)Halko, Martinsson, and Tropp]{Halko2011}
N~Halko, P~G Martinsson, and J~A Tropp.
\newblock {Finding Structure with Randomness: Probabilistic Algorithms for
  Constructing Approximate Matrix Decompositions}.
\newblock \emph{SIAM Review}, 53\penalty0 (2):\penalty0 217--288, 2011.

\bibitem[Held et~al.(2016)Held, Cape, and Tintle]{Held2016}
Elizabeth Held, Joshua Cape, and Nathan Tintle.
\newblock {Comparing machine learning and logistic regression methods for
  predicting hypertension using a combination of gene expression and
  next-generation sequencing data}.
\newblock In \emph{BMC Proceedings}, volume~10, 2016.

\bibitem[Higham(2002)]{Higham2002}
Nicholas~J Higham.
\newblock \emph{{Accuracy and Stability of Numerical Algorithms}}.
\newblock Society for Industrial and Applied Mathematics, 2002.

\bibitem[Kannan and Vempala(2017)]{KannanV17}
Ravindran Kannan and Santosh~S Vempala.
\newblock {Randomized algorithms in numerical linear algebra}.
\newblock \emph{Acta Numerica}, 26:\penalty0 95--135, 2017.

\bibitem[Kumar and Schneider(2017)]{kishore2017literature}
Kishore Kumar and Jan Schneider.
\newblock Literature survey on low rank approximation of matrices.
\newblock \emph{{Linear and Multilinear Algebra}}, 65\penalty0 (11):\penalty0
  2212--2244, 2017.

\bibitem[Li et~al.(2021)Li, Wang, and Woodruff]{li2021tight}
Yi~Li, Ruosong Wang, and David~P Woodruff.
\newblock Tight bounds for the subspace sketch problem with applications.
\newblock \emph{SIAM Journal on Computing}, 50\penalty0 (4):\penalty0
  1287--1335, 2021.

\bibitem[Loh and Wainwright(2013)]{loh2015regularized}
Po-Ling Loh and Martin~J Wainwright.
\newblock Regularized m-estimators with nonconvexity: Statistical and
  algorithmic theory for local optima.
\newblock \emph{Advances in Neural Information Processing Systems}, 26, 2013.

\bibitem[Lozano et~al.(2011)Lozano, Swirszcz, and Abe]{Lozano2011}
Aur{\'{e}}lie~C. Lozano, Grzegorz Swirszcz, and Naoki Abe.
\newblock Group orthogonal matching pursuit for logistic regression.
\newblock In \emph{Proceedings of the Fourteenth International Conference on
  Artificial Intelligence and Statistics}, volume~15, pages 452--460, 2011.

\bibitem[Mai et~al.(2021)Mai, Musco, and Rao]{mai2021coresets}
Tung Mai, Cameron Musco, and Anup Rao.
\newblock Coresets for classification--simplified and strengthened.
\newblock \emph{Advances in Neural Information Processing Systems}, 34, 2021.

\bibitem[Martinsson and Tropp(2020)]{Martinsson2020}
Per~Gunnar Martinsson and Joel~A. Tropp.
\newblock {Randomized numerical linear algebra: Foundations and algorithms}.
\newblock \emph{Acta Numerica}, 29:\penalty0 403--572, 2020.

\bibitem[McCullagh and Nelder(2019)]{mccullagh2019generalized}
Peter McCullagh and John~A Nelder.
\newblock \emph{Generalized linear models}.
\newblock Routledge, 2019.

\bibitem[Munteanu et~al.(2018)Munteanu, Schwiegelshohn, Sohler, and
  Woodruff]{munteanu2018coresets}
Alexander Munteanu, Chris Schwiegelshohn, Christian Sohler, and David Woodruff.
\newblock On coresets for logistic regression.
\newblock \emph{Advances in Neural Information Processing Systems}, 31, 2018.

\bibitem[Munteanu et~al.(2021)Munteanu, Omlor, and
  Woodruff]{pmlr-v139-munteanu21a}
Alexander Munteanu, Simon Omlor, and David Woodruff.
\newblock Oblivious sketching for logistic regression.
\newblock In \emph{Proceedings of the 38th International Conference on Machine
  Learning}, pages 7861--7871, 2021.

\bibitem[Negahban et~al.(2012)Negahban, Ravikumar, Wainwright, and
  Yu]{negahban2012unified}
Sahand~N Negahban, Pradeep Ravikumar, Martin~J Wainwright, and Bin Yu.
\newblock A unified framework for high-dimensional analysis of $ m $-estimators
  with decomposable regularizers.
\newblock \emph{Statistical science}, 27\penalty0 (4):\penalty0 538--557, 2012.

\bibitem[Nesterov(1998)]{nesterov1998introductory}
Yuri Nesterov.
\newblock Introductory lectures on convex programming, 1998.

\bibitem[Paul et~al.(2014)Paul, Boutsidis, Magdon-Ismail, and
  Drineas]{Paul2014}
Saurabh Paul, Christos Boutsidis, Malik Magdon-Ismail, and Petros Drineas.
\newblock {Random Projections for Linear Support Vector Machines}.
\newblock \emph{ACM Transactions on Knowledge Discovery from Data}, 8\penalty0
  (4):\penalty0 1--25, 2014.

\bibitem[Pilanci and Wainwright(2015)]{Pilanci2014}
Mert Pilanci and Martin~J. Wainwright.
\newblock {Randomized sketches of convex programs with sharp guarantees}.
\newblock In \emph{IEEE Transactions on Information Theory}, volume~61, pages
  5096--5115, 2015.

\bibitem[Pilanci and Wainwright(2016)]{Pilanci2016}
Mert Pilanci and Martin~J. Wainwright.
\newblock {Iterative Hessian sketch: Fast and accurate solution approximation
  for constrained least-squares}.
\newblock \emph{Journal of Machine Learning Research}, 17\penalty0
  (53):\penalty0 1--38, 2016.

\bibitem[Priv{\'{e}} et~al.(2019)Priv{\'{e}}, Aschard, and Blum]{Prive2019}
Florian Priv{\'{e}}, Hugues Aschard, and Michael~G.B. Blum.
\newblock {Efficient implementation of penalized regression for genetic risk
  prediction}.
\newblock \emph{Genetics}, 212\penalty0 (1):\penalty0 65--74, 2019.

\bibitem[Sarlos(2006)]{Sarlos2006}
Tamas Sarlos.
\newblock {Improved Approximation Algorithms for Large Matrices via Random
  Projections}.
\newblock In \emph{2006 47th Annual IEEE Symposium on Foundations of Computer
  Science}, pages 143--152, 2006.

\bibitem[Verhulst(1838)]{PFV1838}
Pierre-Fran{\c{c}}ois Verhulst.
\newblock {Notice sur la loi que la population poursuit dans son
  accroissement}.
\newblock \emph{Correspondance Math{\'{e}}matique et Physique}, 10:\penalty0
  113--121, 1838.

\bibitem[Verhulst(1845)]{PFV1845}
Pierre-Fran{\c{c}}ois Verhulst.
\newblock {Recherches math{\'{e}}matiques sur la loi d'accroissement de la
  population}.
\newblock \emph{Nouveaux M{\'{e}}moires de l'Acad{\'{e}}mie Royale des Sciences
  et Belles-Lettres de Bruxelles}, 18:\penalty0 14--54, 1845.

\bibitem[Woodruff(2014)]{Woodruff2014}
David~P. Woodruff.
\newblock {Sketching as a Tool for Numerical Linear Algebra}.
\newblock \emph{Foundations and Trends in Theoretical Computer Science},
  10\penalty0 (1-2):\penalty0 1--157, 2014.

\bibitem[Wu et~al.(2009)Wu, Chen, Hastie, Sobel, and Lange]{Wu2009}
Tong~Tong Wu, Yi~Fang Chen, Trevor Hastie, Eric Sobel, and Kenneth Lange.
\newblock {Genome-wide association analysis by lasso penalized logistic
  regression}.
\newblock \emph{Bioinformatics}, 25\penalty0 (6):\penalty0 714--721, 2009.

\end{thebibliography}
\end{small}

\appendix
\onecolumn
\section{Generative Model} \label{section:generative_model}

We remind the reader that, in the absence of some form of regularization, the worst case forward error for a dimensionally-reduced logistic problem could be unbounded. However, under mild assumptions on the data generating process, we show in this section that this phenomenon will not occur even without any form of regularization. Towards that end, the work of~\citet{loh2015regularized} demonstrated that a simple data generating model is sufficient to ensure that the logistic loss function satisfies a restricted strong convexity condition  with high probability.  This result was later used by~\citet{elenberg2018restricted} to guarantee that the performance of their feature extraction method did not deteriorate in the absence of regularization.  

We now describe the model setup and prove our results. Let the observed samples (the rows of the data matrix $\Xb \in \R^{n \times d}$) be drawn in i.i.d. trials from a zero-mean sub-gaussian distribution with parameter $\sigma_x$ and covariance matrix $\Sigmab$.  Furthermore, let the corresponding labels $\yb_i$ be generated according to the linear log-odds model, by fixing $\beta^*\in \R^d$ to be the true parameter of the logistic regression model and setting $\yb_i$ to one with probability $\nicefrac{1}{1+e^{-\xb_i^T\beta^*}}$ and to $-1$ otherwise.

Let $\beta_1$ and $\beta_2$ be vectors in $\R^d$ and let $\mathbb{B}_p(R) = \{\xb \in \R^n : \|\xb\|_p \leq R\}$. Under this data generating model,~\cite{loh2015regularized} showed that the \emph{Taylor error} around the vector $\beta_2$ in the direction $\beta_1 - \beta_2$, defined as\footnote{We use the notation $\langle \xb,\yb\rangle$ to denote the inner product $\xb^T\yb$ between the two vectors.}
\begin{align*}
\Tcal(\beta_1, \beta_2) = \Lcal(\beta_1) -
\Lcal(\beta_2) - \langle\nabla \Lcal(\beta_2), \beta_1 -
  \beta_2\rangle,
\end{align*}
is upper and lower bounded. We note that the following proposition holds for GLMs beyond logistic regression, where $\alpha_u$ and $\psi(\cdot)$ are defined as in Appendix~\ref{section:glm}; in the special case of logistic regression, $\alpha_u \leq 1$.
\begin{proposition}\label{prop:wainwright} (Proposition 1 in \cite{loh2015regularized})
There exists a constant $\alpha_\ell > 0$, depending only on the GLM and the parameters $(\sigma_x^2, \Sigmab)$, such that for all vectors $\beta_2 \in \mathbb{B}_2(3) \cap \mathbb{B}_1(R)$, we have
\begin{align*}
	\Tcal(\beta_1, \beta_2) \geq 
	 \frac{\alpha_\ell}{2} \|\beta_1 - \beta_2\|_2^2 - \frac{c^2
           \sigma_x^2}{2\alpha_\ell} \frac{\log d}{n} \|\beta_1 - \beta_2\|_1^2,
         \quad \text{for all $\|\beta_1 - \beta_2\|_2 \leq 3$,}  %
\end{align*}
with probability at least $1 - c_1 \exp (-c_2 n)$. With the bound
$\|\psi''\|_\infty \leq \alpha_u$, we also have
\begin{align}
\Tcal(\beta_1, \beta_2) & \leq \alpha_u \lambda_{max}(\Sigmab) \;
\left( \frac{3}{2} \|\beta_1 - \beta_2\|_2^2 + \frac{\log d}{n}
\|\beta_1 - \beta_2\|_1^2\right), \quad \mbox{for all $\beta_1, \beta_2 \in
  \R^d$},
\end{align}
with probability at least $1 - c_1 \exp(-c_2 n)$, where $c_1$ and $c_2$ are fixed constants.
\end{proposition}
We can use the above bound in the same way as we previously used strong convexity from the regularization parameter to bound the forward error of the dimensionally-reduced problem. Again, let $\vb$ be defined as the unit vector pointing from $\beta_d$ to $\Pb_k\beta_k$, i.e.,
\begin{align}
\vb = \frac{1}{\|\Pb_k\beta_k - \beta_d\|_2} \cdot \Pb_k\beta_k - \beta_d.
\end{align}
Let $\beta_2 = \Pb_k\beta_k$ and $\beta_1 = \Pb_k\beta_k - t\cdot \vb$, for some $t>0$. Using Proposition~\ref{prop:wainwright}, we get
\begin{gather*}
    \Lcal(\Pb_k\beta_k - t\cdot\vb) - \Lcal(\Pb_k\beta_k) - \langle\nabla \Lcal(\Pb_k\beta_k), -t\cdot\vb \rangle
    \geq \frac{\alpha_\ell}{2} \|t\cdot\vb\|_2^2 - \frac{c^2\sigma_x^2}{2\alpha_\ell} \frac{\log d}{n} \|t\cdot\vb\|_1^2.
\end{gather*}
We now solve for $t^*$, the minimum value of $t > 0$ that satisfies
\begin{gather*}
    \langle\nabla \Lcal(\Pb_k\beta_k), t^*\cdot\vb \rangle 
    =\frac{\alpha_\ell}{2} \|t^*\cdot\vb\|_2^2 - \frac{c^2\sigma_x^2}{2\alpha_\ell} \frac{\log d}{n} \|t^*\cdot\vb\|_1^2,
\end{gather*}
as this guarantees $\Lcal(\Pb_k\beta_k - t^*\cdot\vb) - \Lcal(\Pb_k\beta_k) \geq 0$.  We next substitute in the simplified gradient (without regularization) from eqn.~(\ref{eqn:pd1}).  Recall that $\Xbbar = -\Db_\yb\Xb$ and $[\wbbar(\beta)]_i = \sigma(\xbbar_i^T\beta)$ to get
\begin{gather*}
    t^*\cdot\vb^T(\Ib - \Pb_k\Pb_k^T) \Xbbar^T \wbbar(\Pb_k\beta_k)
    = \frac{\alpha_\ell}{2} \|t^*\cdot\vb\|_2^2 - \frac{c^2\sigma_x^2}{2\alpha_\ell} \frac{\log d}{n} \|t^*\cdot\vb\|_1^2\\
    \Rightarrow 
     \frac{t^*}{\|\Pb_k\beta_k - \beta_d\|_2} \cdot (-\beta_d^T(\Ib - \Pb_k\Pb_k^T) \Xbbar^T \wbbar(\Pb_k\beta_k))
    = \frac{\alpha_\ell}{2} t^{*2}\cdot\|\vb\|_2^2 - \frac{c^2\sigma_x^2}{2\alpha_\ell} \frac{\log d}{n} t^{*2} \cdot \|\vb\|_1^2.
\end{gather*}
Re-arranging terms, we get:
\begin{gather*}
\|\Pb_k\beta_k - \beta_d\|_2 \cdot t^*=
    \left(\frac{\alpha_\ell}{2}  - \frac{c^2\sigma_x^2}{2\alpha_\ell} \frac{\log d}{n}  \frac{\|\Pb_k\beta_k - \beta_d\|_1^2}{\|\Pb_k\beta_k - \beta_d\|_2^2}\right)^{-1} \cdot   (-\beta_d^T(\Ib - \Pb_k\Pb_k^T) \Xbbar^T \wbbar(\Pb_k\beta_k)).
\end{gather*}
We again use that $\|\Pb_k\beta_k - \beta_d\|_2 \leq t^*$, since the loss decreases monotonically as $\beta$ goes from $\Pb_k\beta_k$  to $\beta_d$, to get:
\begin{gather*}
    \|\Pb_k\beta_k - \beta_d\|_2^2 \leq
    \left(\frac{\alpha_\ell}{2}  - \frac{c^2\sigma_x^2}{2\alpha_\ell} \frac{\log d}{n}  \frac{\|\Pb_k\beta_k - \beta_d\|_1^2}{\|\Pb_k\beta_k - \beta_d\|_2^2}\right)^{-1} \cdot  (-\beta_d^T(\Ib - \Pb_k\Pb_k^T) \Xbbar^T \wbbar(\Pb_k\beta_k)).
\end{gather*}
We note that the ratio, $$\frac{\|\Pb_k\beta_k - \beta_d\|_1}{\|\Pb_k\beta_k - \beta_d\|_2,}$$ is upper bounded by the so-called \textit{subspace compatibility constant} of the range of $\Pb_k$, as defined by~\citet{negahban2012unified}.  For large enough $n$, relative to the subspace compatibility factor, there exists a universal constant $\tau > 0$ such that the following holds:
\begin{align*}
    \|\Pb_k\beta_k - \beta_d\|_2^2 
    &\leq \frac{-1}{\tau}\beta_d^T(\Ib - \Pb_k\Pb_k^T) \Xbbar^T \wbbar(\Pb_k\beta_k) \\
    &= \frac{1}{\tau}\beta_d^T(\Ib - \Pb_k\Pb_k^T) \Xb^T\Db_\yb \wb(\Pb_k\beta_k).
\end{align*}
Note that the above upper bound holds without any regularization, i.e., $\lambda = 0$.  

We now proceed to provide a lower bound for the forward error under the aforementioned generative model. Using Proposition~\ref{prop:wainwright}, we get the following upper bound on $\Tcal(\Pb_k\beta_k - t\cdot\vb, \Pb_k\beta_k)$; again note that for the special case of logistic regression $\alpha_u \leq 1$:
\begin{align*}
    \Lcal(\Pb_k\beta_k - t\cdot\vb) - \Lcal(\Pb_k\beta_k) - \langle\nabla \Lcal(\Pb_k\beta_k), -t\cdot\vb \rangle
    &\leq \lambda_{max}(\Sigmab) \;
    \left( \frac{3}{2} \|t\cdot \vb\|_2^2 + \frac{\log d}{n} \|t\cdot \vb\|_2^2\right) \\
    &= t^2 \cdot \lambda_{max}(\Sigmab)\left(\frac{3}{2} + \frac{\log d}{n}\right).
\end{align*}
Moving the gradient term to the right side and setting $\alpha = 2\lambda_{max}(\Sigmab)\left(\frac{3}{2} + \frac{\log d}{n}\right)$ results in a formula with the same form as eqn.~(\ref{eq:intermediate_upper}):
\begin{gather*}
    \Lcal_d(\Pb_k\beta_k - t \cdot \vb) - \Lcal_d(\Pb_k\beta_k)
    \leq t\cdot (-\vb^T\nabla \Lcal_d(\Pb_k\beta)) + \frac{t^2 \cdot \alpha}{2}, \quad
    \text{for all }t \in [0,1].
\end{gather*}
From here on, we can follow the remainder of the proof in Section~\ref{sxn:forward_lower}, starting with eqn.~(\ref{eq:intermediate_upper}) and using $\alpha = 2\lambda_{max}(\Sigmab)\left(\frac{3}{2} + \frac{\log d}{n}\right)$. We eventually conclude that:
\begin{gather*}
    \|\Pb_k\beta_k - \beta_d\|_2^2 
    \geq \left(2\lambda_{max}(\Sigmab)\left(\frac{3}{2} + \frac{\log d}{n}\right)\right)^{-1} \cdot (- \beta_d^T (\Ib - \Pb_k\Pb_k^T) \Xbbar^T  \wbbar(\Pb_k\beta_k)).
\end{gather*}
For $n \geq 2\log d$, the above equation simplifies to:
\begin{align*}
    \|\Pb_k\beta_k - \beta_d\|_2^2 
    &\geq \frac{-1}{4\lambda_{max}(\Sigmab)} \cdot  \beta_d^T (\Ib - \Pb_k\Pb_k^T) \Xbbar^T  \wbbar(\Pb_k\beta_k) \\
    &= \frac{1}{4\lambda_{max}(\Sigmab)} \cdot  \beta_d^T (\Ib - \Pb_k\Pb_k^T) \Xb^T\Db_\yb  \wb(\Pb_k\beta_k).
\end{align*}
We are now ready to summarize our results for the generative model of~\cite{loh2015regularized} in the following lemma, which provides (almost) tight upper and lower bounds for the forward error for non-reguralized logistic regression.
\begin{lemma}\label{lemma:appA2}
If $\Pb_k\beta_k \in \mathbb{B}_2(3) \cap \mathbb{B}_1(R)$ and $\|\Pb_k\beta_k - \beta_d\|_2 \leq 3$, then for $n \geq \Ocal(\log d)$, there exists a constant, $\tau$, depending only on $(\sigma_x^2, \Sigmab)$ and the subspace compatibility constant of the range of $\Pb_k$\footnote{The $\ell_1$-$\ell_2$ \textit{subspace compatibility constant} defined by~\cite{negahban2012unified} for a subspace $\Mcal$ is $\Psi(\Mcal) = \sup_{\ub \in \Mcal} \frac{\|\ub\|_1}{\|\ub\|_2}$.}
such that
\begin{align*}
    \|\Pb_k\beta_k - \beta_d\|_2^2 
    &\leq \frac{1}{\tau} \cdot \beta_d^T(\Ib - \Pb_k\Pb_k^T) \Xb^T\Db_\yb \wb(\Pb_k\beta_k), \quad \text{and} \\
    \|\Pb_k\beta_k - \beta_d\|_2^2 
    &\geq \frac{1}{4\lambda_{max}(\Sigmab)} \cdot  \beta_d^T (\Ib - \Pb_k\Pb_k^T) \Xbbar^T\Db_\yb  \wb(\Pb_k\beta_k),
\end{align*}
hold with probability at least $1- 2c_1\exp(-c_2 n)$, where $c_1$ and $c_2$ are fixed constants.
\end{lemma}
Note that the upper and lower bounds are tight up to constants and a dependency on the largest singular value of the covariance matrix $\Sigmab$.

\section{Extension to GLMs}\label{section:glm}

\subsection{Forward error bound for GLMs}\label{sxn:glm_upper}

In this section, we show how our results can be generalized to hold for Generalized Linear Models (GLMs) beyond logistic regression.  We start by writing down the general formulas for the conditional distribution and prediction of an arbitrary GLM with linear parameter $\beta \in \R^d$, scale parameter $\sigma >0$, and cumulant function $\psi$. We follow the lines of~\citet{mccullagh2019generalized} to get:
\begin{align*}
    \PP(\yb_i | \xb_i, \beta, \sigma) &= \exp\left\{\frac{\yb_i\xb_i^T\beta - \psi(\xb_i^T\beta)}{c(\sigma)}\right\}, \quad \text{and}\\
    \EE[\yb_i | \xb_i, \beta, \sigma] &= \psi'(\xb_i^T\beta).
\end{align*}
The empirical (un-normalized) log-loss function is given by:
\begin{align*}
    \Lcal_\text{log-loss}(\beta) = \sum_{i=1}^n \frac{-1}{c(\sigma)}(\yb_i\xb_i^T\beta - \psi(\xb_i^T\beta)).
\end{align*}
We can now define the original and dimensionally-reduced loss functions with $\ell_2^2$-regularization as follows: set $c(\sigma) = 1$ (without loss of generality):
\begin{align*}
    \beta_d &= \argmin_{\beta \in \R^d} \Lcal_d(\beta) = \argmin_{\beta \in \R^d} \sum_{i=1}^n -(\yb_i\xb_i^T\beta - \psi(\xb_i^T\beta)) + \frac{\lambda}{2}\|\beta\|_2^2,\quad \text{and} \\
    \beta_k &= \argmin_{\beta \in \R^k} \Lcal_k(\beta) = \argmin_{\beta \in \R^k} \sum_{i=1}^n -(\yb_i\xb_i^T\Pb_k\beta - \psi(\xb_i^T\Pb_k\beta)) + \frac{\mu}{2} \|\beta\|_2^2.
\end{align*}
We note that the above two equations are analogs of eqns.~(\ref{eqn:blrr}) and~(\ref{eqn:sketchedregression1}). The corresponding gradients of the above loss functions are given by
\begin{align*}
    \nabla \Lcal_d(\beta) &= \sum_{i=1}^n (\psi'(\xb_i^T\beta) - \yb_i)\xb_i + \lambda\cdot\beta, \quad \text{and}\\
    \nabla \Lcal_k(\beta) &= \sum_{i=1}^n (\psi'(\xb_i^T\Pb_k\beta) - \yb_i)\Pb_k^T\xb_i + \mu\cdot\beta.
\end{align*}
We now define $\wbbar(\beta) \in \R^n$ such that $[\wb(\beta)]_i = (\psi'(\xb_i^T\beta) - \yb_i)$.  We can now rewrite the above two gradients using matrix notation as
\begin{align*}
    \nabla \Lcal_d(\beta) &=\Xb^T\wbbar(\beta) + \lambda\cdot\beta, \quad \text{and}\\
    \nabla \Lcal_k(\beta) &= \Pb_k^T\Xb^T\wbbar(\Pb_k\beta) + \mu\cdot\beta.
\end{align*}
Notice that the above two formulas for the gradients are equivalent to eqn.~(\ref{eqn:grall}) except for the definition of $\wbbar(\cdot)$.  From here on, we can closely follow the remainder of the proof of Theorem~\ref{thm:forward_error_upper_bound} in Section~\ref{sxn:forward_upper} to get an analogous upper bound on the forward error. If we let $\wb(\beta) = -1\cdot \wbbar(\beta)$ and $\mu = \lambda$, then the final error bound is
\begin{align*}
    \|\Pb_k\beta_k - \beta_d\|_2^2 
    &\leq \frac{2}{\lambda}\beta_d^T(\Ib - \Pb_k\Pb_k^T) \Xb^T \wb(\Pb_k\beta_k).
\end{align*}

\subsection{Lower bound for GLMs}

To extend our lower bounds to GLMs, let $\Db(\beta) \in \R^{n \times n}$ be a diagonal matrix whose $i$-th entry  for $i=1\ldots n$ is $\psi''(\xb_i^T\beta)$.  The Hessian of $\Lcal_d(\beta)$ is given by:
\begin{align*}
    \nabla^2\Lcal_d(\beta) &= \Xb^T\Db(\beta)\Xb + \lambda \cdot \Ib.
\end{align*}
In Section~\ref{sxn:forward_lower} we used the fact that $\|\Db(\beta)\|_2 \leq \nicefrac{1}{4}$ for the special case of logistic regression.  We now assume there exists a constant $\alpha_u > 0$ such that 
\begin{align}
\psi''(t) \leq \alpha_u  \label{eqn:GLMassume}
\end{align}
for all $t \in \R$. As discussed in Section 3.3 of~\citet{loh2015regularized}, this assumption holds for many GLMs, such as logistic regression, linear regression, and multinomial regression, but it does not hold in other cases, like Poissonn regression.  Using this assumption, we can bound the spectral norm of the Hessian:
\begin{gather*}
    \|\nabla^2 \Lcal_d(\beta)\|_2 
    = \|\Xb^T\Db(\beta)\Xb + \lambda \cdot \Ib\|_2
    \leq \|\Xb^T\|_2\|\Db(\beta)\|_2\|\Xb\|_2 + \lambda\|\Ib\|_2
    = \alpha_u \|\Xb\|_2^2 + \lambda.
\end{gather*}
Therefore, the log-loss of the GLM is $\alpha$-smooth, where $\alpha = \alpha_u \|\Xb\|_2^2 + \lambda$.  We can then follow the rest of the proof in Section \ref{sxn:forward_lower} starting from eqn. (\ref{eq:intermediate_upper}) using the definitions of $\wb(\cdot)$ and $\wbbar(\cdot)$ in Section \ref{sxn:glm_upper} to get:
\begin{align*}
    \|\Pb_k\beta_k - \beta_d\|_2^2 
    &\geq \frac{1}{\alpha} \cdot  \beta_d^T (\Ib - \Pb_k\Pb_k^T) \Xb^T  \wb(\Pb_k\beta_k) \\
    &= \frac{1}{\alpha_u\|\Xb\|_2^2 + \lambda} \cdot  \beta_d^T (\Ib - \Pb_k\Pb_k^T) \Xb^T  \wb(\Pb_k\beta_k).
\end{align*}
Since $\alpha_u$ is a constant that only depends on the chosen GLM, we again conclude that, when the two-norm of the input matrix is constant, our upper and lower bounds differ by a factor of $\Ocal(\nicefrac{1}{\lambda})$.

\section{Information-theoretic bound on forward error}\label{section:information_theoretic_bound}

In this section, we show that our upper bound in Theorem \ref{thm:forward_error_upper_bound} implies an upper bound on the forward error of a logistic regression problem in terms of information theoretic quantities under standard statistical assumptions of logistic regression~\citep[Chapter 4]{mccullagh2019generalized}. This probabilistic interpretation requires no generative model assumption on the data points $(\yb_i, \xb_i)$.  It simply interprets the vectors $(\Ib - \Pb_k\Pb_k^T)\beta_d$ and $\Pb_k\beta_k$ as encoding distributions on the space of labels as described by the linear log-odds model.  

For a fixed point $\xb \in \R^d$, its corresponding label has the following distribution under the linear log-odds model with parameter $\beta$:
\begin{gather*}
    \PP\left(\yb = 1 | \xb, \beta\right) = \frac{1}{1+e^{-\xb^T\beta}}
\end{gather*}

For a fixed $\beta \in \R^d$, we can define another induced binary distribution that encodes the distribution of correct.  Let 
\begin{gather}
    q_i = \PP(\yb \neq \yb_i | \xb_i, (\Ib - \Pb_k\Pb_k^T)\beta_d) =  \frac{1}{1+e^{\yb_i\cdot\xb_i^T(\Ib - \Pb_k\Pb_k^T)\beta_d}} \label{eq:logistic_misclassification_probability_sketch}\\
    \text{and}\nonumber\\
    p_i = \PP(\yb \neq \yb_i | \xb_i, \Pb_k\beta_k) =  \frac{1}{1+e^{\yb_i\cdot\xb_i^T\Pb_k\beta_k}}. \label{eq:logistic_misclassification_probability_original} 
\end{gather}
Then, $q_i$ is the probability that a draw from the linear log-odds model at $\xb_i$ parameterized by $\beta = (\Ib - \Pb_k\Pb_k^T)\beta_d$ will not match the true observed label $\yb_i$.  The same intuition holds for $p_i$ except with $\beta = \Pb_k\beta_k$.  Therefore, we can denote the set of events $\{M, N\}$ called ``Match'' and ``No Match'' which indicate whether a draw from a linear log-odds model matches the true label $\yb_i$.  For a fixed tuple $(\xb_i, \yb_i, \beta_d, \beta_k, \Pb_k)$, we have two well defined distributions on $\{M,N\}$. Interestingly, the $\ell_2$-norm difference between $\beta_d$ and $(\Ib - \Pb_k\Pb_k^T)\beta_k$ is bounded by the cross-entropy of these two distributions.
\begin{theorem}
    Let $\Xb \in \R^{n \times d}$, $\yb \in \R^n$, $\beta_d \in \R^d$, and $\lambda > 0$ define the original logistic regression problem as defined in Section \ref{section:contributions}. Let $\Pb_k \in \R^{k \times d}$ be a sketching matrix such that $\Pb_k\Pb_k^T$ is a projection matrix, and let $\beta_k$ be the optimal solution of the sketched problem.  Let $\Dcal_q$ and $\Dcal_p$ be distributions on $\{-1, 1\}^n$ where the distribution on the $i$-th index is defined by eqns. \ref{eq:logistic_misclassification_probability_sketch} and \ref{eq:logistic_misclassification_probability_original} respectively. Then,
    \begin{gather*}
        \|\beta_d - \Pb_k\beta_k\|_2^2 \leq H(\Dcal_p, \Dcal_q),
    \end{gather*}
    where $H(\cdot, \cdot)$ denotes the cross-entropy between the two distributions.
\end{theorem}
\begin{proof}
By Theorem \ref{thm:forward_error_upper_bound}:
\begin{align*}
        \|\Pb_k\beta_k - \beta_d\|_2^2
        &\leq \frac{2}{\lambda} \sum_{i=1}^n \yb_i\cdot\beta_\lambda^T(\Ib - \Pb_k\Pb_k^T) \xb_i \cdot \sigma(-\yb_i\xb_i^T\Pb_k\beta_k),
    \end{align*}
We can rewrite a single summand in the bound by substituting in eqns. \ref{eq:logistic_misclassification_probability_sketch} and \ref{eq:logistic_misclassification_probability_original}:
\begin{align*}
    \yb_i\cdot\beta_\lambda^T(\Ib - \Pb_k\Pb_k^T) \xb_i \cdot \sigma(-\yb_i\xb_i^T\Pb_k\beta_k)
    &= \log \exp(\yb_i\cdot\beta_\lambda^T(\Ib - \Pb_k\Pb_k^T) \xb_i ) \cdot \frac{1}{1+e^{\yb_i\xb_i^T\Pb_k\beta_k}} \\
    &= \log \left(\frac{1}{q_i} - 1\right) \cdot p_i \\
    &= \log \left(\frac{1-q_i}{q_i}\right) \cdot p_i \\
    &= - p_i \log \frac{q_i}{1-q_i} \\
    &= -p_i \log q_i + p_i \log ( 1 - q_i) \\
    &= -p_i \log q_i - (1 - p_i) \log ( 1 - q_i) + \log(1-q_i) \\
    &\leq -p_i \log q_i - (1 - p_i) \log ( 1 - q_i).
\end{align*}

Let there exist two distributions with binary outcomes on the same event space.  If the probability of event one occurring is $p$ in the first distribution and $q$ in the second, the the cross-entropy of the first distribution relative to the second is $H(p,q) = -p\log q - (1-p) \log (1-q)$.  Therefore, we see that each summand of our forward error bound is bounded by the binary cross-entropy $H(p_i, q_i)$. Our new total bound can be written as:
\begin{align*}
        \|\Pb_k\beta_k - \beta_d\|_2^2
        &\leq \frac{2}{\lambda} \sum_{i=1}^n H(p_i, q_i)
\end{align*}
The binary cross entropy can be rewritten as $H(p, q) = H(p) + \dkl(p, q)$, where $\dkl$ denotes KL-divergence.  Both entropy and KL-divergence decompose additively over a product distribution of independent marginal distributions \cite{cover1999elements}. If $\qb \sim \Dcal_q$ and $\pb \sim \Dcal_p$, each coordinate of these vectors is independent and the $i$-th coordinates are distributed as Bernoulli random variable with parameters $q_i$ and $p_i$ respectively. Therefore,  
\begin{gather*}
    \|\Pb_k\beta_k - \beta_d\|_2^2 \leq \frac{2}{\lambda} H(\Dcal_p, \Dcal_q).
\end{gather*}

\end{proof}
This bound is complementary to Theorem \ref{thm:forward_error_upper_bound} in that it intuitively captures the same relationship, but sacrifices tightness of the bound for interpretability. Fundamentally, the forward error must depend on the true labels of the data in a non-linear and non-smooth manner. The cross-entropy of the the modeled labels under the two models is one more familiar way to represent this relationship. Alternatively, one may view this result as a lower bound on the cross-entropy of the modeled labels in terms of the $\ell_2$-norm distance between the parameters of the original and sketched problem.

\section{Proofs for Section \ref{sxn:logistic_loss}}\label{sxn:proof_appendix}

\textbf{Proof of Theorem \ref{thm:space_lower_bound}}

\begin{proof}
Define $\Rcal(\beta) = \sum_{i=1}^n \max\{0, \xb_i^T\beta\}$, i.e., the ReLu loss. We will first lower bound the space needed by any data structure which approximates ReLu loss to $\epsilon$-relative error. Later, we will show that this implies a lower bound on the space complexity of any data structure $f(\cdot)$ for approximating logistic loss. Let $g(\cdot)$ approximate $\Rcal(\cdot)$ such that $(1-\epsilon) \Rcal(\beta) \leq g(\beta) \leq (1+\epsilon)\Rcal(\beta)$ for all $\beta \in \R^d$. We can rewrite $\Rcal(\beta)$ as follows:
\begin{flalign*}
    \Rcal(\beta) &= \sum_{i=1}^n \max\{0, \xb_i^T\beta\} 
    = \sum_{i=1}^n \nicefrac{1}{2}\cdot \xb_i^T\beta + \nicefrac{1}{2}\cdot |\xb_i^T\beta| 
    = \frac{1}{2}\one^T \Xb\beta + \frac{1}{2}\|\Xb\beta\|_1.
\end{flalign*}
We next use that $\Rcal(\beta) \leq \|\Xb\beta\|_1$.
\begin{flalign*}
    |\Rcal(\beta) - g(\beta)|
    = |\frac{1}{2}\one^T \Xb\beta + \frac{1}{2}\|\Xb\beta\|_1 - g(\beta)| \leq \epsilon \Rcal(\beta) \\
    \Rightarrow |\frac{1}{2}\one^T \Xb\beta + \frac{1}{2}\|\Xb\beta\|_1 - g(\beta)| \leq \epsilon\|\Xb\beta\|_1.
\end{flalign*}

We can store $\one^T\Xb$ exactly in $\Ocal(d)$ space as a length $d$ vector.  We define a new function $h(\beta) = 2g(\beta) - \one^T\Xb\beta$, and by the above inequality, $h(\beta)$ satisfies $|\|\Xb\beta\|_1 - h(\beta)| \leq 2\epsilon\|\Xb\beta\|_1$ for all $\beta \in \R^d$.  Therefore, $h(\beta)$ is an $\epsilon$-relative approximation to $\|\Xb\beta\|_1$ after adjusting for constants and solves the $\ell_1$-subspace sketch problem (see Definition 1.1 in \cite{li2021tight}). By Theorem 1.2 in \cite{li2021tight}, the data structure $h(\cdot)$ requires $\Omega\left(\frac{d}{\epsilon^2 \operatorname{polylog}(\epsilon^{-1})}\right)$ bits of space  if $d = \Omega(\log 1/\epsilon)$ and $n = \Omegatil\left(d\epsilon^{-2}\right)$. Therefore, we conclude that any data structure which approximates $\Rcal(\beta)$ to $\epsilon$-relative error for all $\beta \in \R^d$ must use $\Omegatil\left(\frac{d}{\epsilon^2}\right)$ bits in the worst case.

Next, we show that a data structure, $f(\cdot)$, which approximates logistic loss to relative error can be used to construct an approximation to the ReLu loss, $g(\cdot)$, by $g(\beta) = \frac{1}{t} \cdot f(t\cdot \beta)$ for large enough constant $t > 0$.  To show this, we first bound the approximation error of the logistic loss for a single point.  First, we derive the following inequality when $r>0$.
\begin{gather*}
    \frac{1}{t}\log(1 + e^{t\cdot r}) - r = \frac{1}{t}\left(\log(e^{rt}) + \log\left(\frac{1 + e^{rt}}{e^{rt}}\right)\right) - r
    = \frac{1}{t}\cdot\log\left(1 + \frac{1}{e^{rt}}\right) \leq \frac{1}{t \cdot e^{rt}}.
\end{gather*}
Therefore, if $\xb_i^T\beta > 0$, then $|\frac{1}{t}\log(1 + e^{t\cdot\xb_i^T\beta}) - \xb_i^T\beta| < \frac{1}{t\cdot e^{t\cdot\xb_i^T\beta}}$.  Next, we consider the case where $r \leq 0$. For the case $r < 0$ (in which case $\operatorname{ReLu}(r) = 0$), it directly follows that $\frac{1}{t}\log(1 + e^{t\cdot r})\leq \frac{e^{t\cdot r}}{t}$.  Therefore,
\begin{gather*}
    |\frac{1}{t}\cdot\log(1+e^{t\cdot r}) - \max\{0, r\}| \leq \frac{1}{t \cdot e^{t\cdot|r|}} \leq \frac{1}{t}.
\end{gather*}
We use this inequality to bound the difference in the transformed logistic loss and ReLu loss as follows.
\begin{gather}\label{eqn:logistic_to_ReLu}
    |\frac{1}{t}\cdot\Lcal(t\cdot\beta) - \Rcal(\beta)| = \Big|\sum_{i=1}^n \frac{1}{t}\log(1+e^{t\cdot\xb_i^T\beta}) - \max\{0, \xb_i^T\beta\}\Big|
    \leq \sum_{i=1}^n |\frac{1}{t} \log(1+e^{t\cdot\xb_i^T\beta}) - \max\{0, \xb_i^T\beta\}| \leq \frac{n}{t}.
\end{gather}
Therefore, if we set $t = \frac{n}{\epsilon \cdot \Rcal(\beta)}$, then $|\frac{1}{t} \Lcal(t\cdot \beta) - \Rcal(\beta)| \leq \epsilon\Rcal(\beta)$ for all $\beta \in \R^d$. However, we don't know $\Rcal(\beta)$ exactly, and furthermore, it is possible that $\Rcal(\beta) = 0$.  To handle these issues, we first observe that, for fixed dimensions $n$ and $d$, the set of possible input $(\Xb, \yb)$ is finite due to the bounded bit complexity of entries in $\Xb$.  Therefore, with bounded time complexity, we can compute $\Rcal_{\operatorname{min}}(\beta) = \inf_{\Xb, \yb : \Rcal(\beta; \Xb, \yb) > 0} \Rcal(\beta)$, that is, the minimum positive value of $\Rcal(\beta)$ over all possible input of a fixed dimension.  We then set $t = \frac{4n}{\epsilon \cdot \Rcal_{\operatorname{min}}(\beta)}$. By definition $\Rcal_{\min}(\beta)$ is the smallest possible non-zero value of $\Rcal(\beta)$, hence there are two possible cases 1) $\Rcal(\beta) \geq \Rcal_{\min}(\beta)$ and 2) $\Rcal(\beta) = 0$. First, we show that returning $\frac{1}{t}\Lcal(t\cdot\beta)$ gives an $\epsilon$-relative error approximation in the first case:
\begin{gather}\label{eq:relu_case_one}
    \Rcal(\beta) \geq \Rcal_{\min}(\beta) 
    \Rightarrow |\frac{1}{t} \Lcal(t\cdot \beta) - \Rcal(\beta)| \leq \frac{\epsilon}{4}\Rcal_{\min}(\beta) \leq \epsilon\Rcal(\beta).
\end{gather}
Next, we show that if we are in the second case, then $\frac{1}{t}\Lcal(t\cdot\beta)$ is sufficiently smaller than $\Rcal_{\min}(\beta)$ so that we can certify that $\Rcal(\beta) = 0$, in which case we can just return zero to get no error in approximating $\Rcal(\beta)$.
\begin{gather}\label{eq:relu_case_two}
    \Rcal(\beta) = 0
    \Rightarrow |\frac{1}{t} \Lcal(t\cdot \beta)| \leq \frac{\epsilon}{4}\Rcal_{\min}(\beta) 
    \Rightarrow |\frac{1}{t} \Lcal(t\cdot \beta) - \Rcal_{\min}(\beta)| > \frac{\epsilon}{4} \Rcal_{\min}(\beta)
\end{gather}

Above, we showed that $\Lcal(\beta)$ could be used to provide an $\epsilon$-relative error approximation to $\Rcal(\beta)$. Now, we show that if $f(\beta)$ is an $\epsilon$-relative error approximation of $\Lcal(\beta)$, then $g(\beta) = \frac{1}{t}\cdot f(t\cdot\beta)$ is a $3\epsilon$-relative error approximation to $\Rcal(\beta)$ for the value of $t$ defined above. First, we derive the following inequality using eqn. (\ref{eq:relu_case_one}) and the error guarantee of $f(\cdot)$.
\begin{gather*}
    \Big|\frac{1}{t} \Lcal(t \cdot \beta) - \frac{1}{t} f(t \cdot \beta)\Big|
    \leq \epsilon \cdot \frac{1}{t} \cdot \Lcal(t \cdot \beta)
    \leq \epsilon \cdot \left(\Rcal(\beta) + \frac{\epsilon}{4} \Rcal_{\min}(\beta)\right)
\end{gather*}
If $\Rcal(\beta) > 0$, then $\Rcal(\beta) \geq \Rcal_{\min}(\beta)$, and therefore we conclude $|\frac{1}{t} \Lcal(t \cdot \beta) - \frac{1}{t} f(t \cdot \beta)| \leq 2\epsilon\Rcal(\beta)$ from the previous equation.
We now use this result along with eqn. (\ref{eq:relu_case_one}) to prove that $g(\beta)$ is a $3\epsilon$-relative error approximation of $\Rcal(\beta)$ when $\Rcal(\beta) > 0$.
\begin{align*}
    |\frac{1}{t}\cdot f(t\cdot\beta) - \Rcal(\beta)|
    \leq |\frac{1}{t}\cdot f(t\cdot\beta) - \frac{1}{t}\cdot\Lcal(t\cdot \beta)| + |\frac{1}{t}\cdot \Lcal(t\cdot\beta) - \Rcal(\beta)| 
    \leq 3\epsilon \Rcal(\beta).
\end{align*}
Alternatively, if $\Rcal(\beta) = 0$, then:
\begin{gather*}
    \Big|\frac{1}{t} \Lcal(t \cdot \beta) - \frac{1}{t} f(t \cdot \beta)\Big|
    \leq \frac{\epsilon^2}{4}\Rcal_{\min}(\beta) 
    \quad\text{ and }\quad
    |\frac{1}{t} \Lcal(t\cdot \beta)| \leq \frac{\epsilon}{4}\Rcal_{\min}(\beta)\\
    \Rightarrow \Big| \frac{1}{t} f(t\cdot \beta) \Big| \leq \frac{\epsilon}{2}\Rcal_{\min} 
    \Rightarrow \Big|\frac{1}{t} f(t\cdot \beta) - \Rcal_{\min}(\beta) \Big| > \frac{\epsilon}{4}\Rcal(\beta).
\end{gather*}
Therefore, if $\frac{1}{t}f(t\cdot \beta) \leq \frac{\epsilon}{2}\Rcal_{\min}(\beta)$, then $\Rcal(\beta)$ must equal zero, and so we can return zero to get zero approximation error. After adjusting $\epsilon$ by a factor of three in the above proof, we conclude $f(\cdot)$ must use $\Omegatil(\frac{d}{\epsilon^2})$ bits of memory in the worst case.
\end{proof}

\textbf{Proof of Theorem \ref{thm:logistic_loss_upper}}
\begin{proof}

To simplify the notation, let $\sbb = \Xb\beta$ and $\db = (\Xb - \Xbtil)\beta$.  We can then write the difference in the log loss as:
\begin{flalign*}
    |\Lcal(\beta; \Xb) - \Lcal(\beta, \Xbtil)|
    &= \left(\sum_{i=1}^n \log\left( 1 + e^{\xb_i^T\beta}\right) + \frac{\lambda}{2} \|\beta\|_2^2\right) - \left(\sum_{i=1}^n \log\left( 1 + e^{\xbtil_i^T\beta}\right) + \frac{\lambda}{2} \|\beta\|_2^2\right) \\
    &=\left|\sum_{i=1}^n \log \left(\frac{1 + e^{\sbb_i}}{1 + e^{\sbb_i + \db_i}}  \right)\right| \\
    &\leq \left|\sum_{i=1}^n \log \left(\frac{1 + e^{\sbb_i}}{1 + e^{\sbb_i - |\db_i|}}  \right)\right| \\
    &= \left|\sum_{i=1}^n \log \left(\frac{1 + e^{\sbb_i}}{1 + e^{-|\db_i|}e^{\sbb_i}}  \right)\right| \\
    &\leq \left|\sum_{i=1}^n \log \left(\frac{1}{e^{-|\db_i|}}\frac{1 + e^{\sbb_i}}{1 + e^{\sbb_i}}  \right)\right| \\
    &= \sum_{i=1}^n  |\db_i| = \|\db\|_1.
\end{flalign*}

Therefore, we can conclude that $|\Lcal(\beta; \Xb) - \Lcal(\beta; \Xbtil)| \leq \|\db\|_1 \leq \sqrt{n}\|\db\|_2 \leq \sqrt{n} \|\Xbtil - \Xb\|_2 \|\beta\|_2$.

\end{proof}

\textbf{Proof of Theorem \ref{thm:logistic_loss_lower}}
\begin{proof}
To prove the theorems statement, we first consider the case of square matrices $(d = n)$.  In particular, first consider the case where $d = n = 1$, where $\Xb = [x]$ and $\Xbtil = [x + s]$, in which case $\|\Xb - \Xbtil\|_2 = s$.  Then,
\begin{gather*}
     \lim_{x \rightarrow \infty} \Lcal([1]; [x+s]) - \Lcal([1]; [x]) = \lim_{x \rightarrow \infty} \log(1+e^{x + s}) - \log(1 + e^x)  = s
\end{gather*}
Which shows that for $\beta = [1]$ and $x$ with large enough magnitude $\Lcal(\beta; \Xbtil) - \Lcal(\beta; \Xb) = (1-\delta)\|\Xb - \Xbtil\|_2$. Next, let $\Xb = x \cdot \Ib_n$, $\Xbtil = (x+s) \cdot \Ib_n$, and $\beta = \one_n$.  Then for all $i \in [n]$, $\xb_i^T\beta = x$ and $\xb_i^T\beta = x+s$.  Therefore, 
\begin{gather*}
     \lim_{x \rightarrow \infty} \Lcal(\beta; \Xbtil) - \Lcal(\beta; \Xb) = \lim_{x \rightarrow \infty} \sum_{i=1}^n \left[\log(1+e^{x + s}) - \log(1 + e^x)\right] = sn.
\end{gather*}
Since $\|\Xb - \Xbtil\|_2 = \|s \cdot \Ib\|_2 = s$.  $\|\beta\|_2 = \sqrt{n}$,
\begin{gather*}
    \lim_{x \rightarrow \infty} \Lcal(\beta; \Xbtil) - \Lcal(\beta; \Xb)
    = sn = \sqrt{n} \|\Xb - \Xbtil\|_2 \|\beta\|_2
\end{gather*}
Hence, we conclude the statement of the theorem for the case where $d = n$. To conclude the case for $d \geq n$, note that $\sqrt{n}\|\Xb - \Xbbar\|_2\|\beta\|_2$ does not change if we extend $\Xb$ and $\Xbbar$ with columns of zeroes and extend $\beta$ with entries of zero until $\Xb,\Xbbar \in \R^{n \times d}$ and $\R^d$.  This procedure also does not change the loss at $\beta$, hence we conclude the statement of the theorem.
\end{proof}

\end{document}